\documentclass{article}

\usepackage{microtype}
\usepackage{graphicx}
\usepackage{booktabs} % for professional tables

\usepackage{hyperref}

\usepackage[accepted]{icml2024}

\usepackage{amsmath}
\usepackage{amssymb}
\usepackage{mathtools}
\usepackage{amsthm}

\usepackage[capitalize,noabbrev]{cleveref}

\usepackage[longend,ruled,vlined,linesnumbered]{algorithm2e}
\usepackage{subfig}
\usepackage{bbm}
\usepackage{bm}
\usepackage{makecell}
\usepackage{multirow}
\usepackage{tikz-cd}
\usepackage{soul}
\usepackage[export]{adjustbox}

\theoremstyle{plain}
\newtheorem{theorem}{Theorem}[section]

\newtheorem{lemma}[theorem]{Lemma}

\theoremstyle{definition}
\newtheorem{definition}[theorem]{Definition}

\newtheorem{example}[theorem]{Example}
\theoremstyle{remark}

\newcommand{\test}{\bm{t}}
\newcommand{\idata}{\mathcal{D}^\square}

\newcommand{\prob}{\mathbb{P}}

\newcommand{\NP}{\mathbf{NP}}

\newcommand{\crnbc}{\mathsf{CR\text{-}NaiveBayes}}

\newcommand{\poisoncrnbc}{\mathsf{CR\text{-}NaiveBayes}^{\dagger}}

\newcommand{\support}[3]{S_{{#1}}({#2} \mid {#3})}
\newcommand{\minsupport}[3]{S_{{#1}}^{\downarrow}({#2} \mid {#3})}
\newcommand{\maxsupport}[3]{S_{{#1}}^{\uparrow}({#2} \mid {#3})}

\usepackage[textsize=tiny]{todonotes}

\icmltitlerunning{Naive Bayes Classifiers over Missing Data: Decision and Poisoning}

\begin{document}

\twocolumn[
\icmltitle{Naive Bayes Classifiers over Missing Data: Decision and Poisoning}



\icmlsetsymbol{equal}{*}

\begin{icmlauthorlist}
\icmlauthor{Song Bian}{equal,yyy}
\icmlauthor{Xiating Ouyang}{equal,yyy}
\icmlauthor{Zhiwei Fan}{yyy}
\icmlauthor{Paraschos Koutris}{yyy}
\end{icmlauthorlist}

\icmlaffiliation{yyy}{Department of Computer Sciences, University of Wisconsin-Madison, Madison WI, USA}

\icmlcorrespondingauthor{Song Bian}{sbian8@wisc.edu}

\icmlkeywords{Machine Learning, ICML}

\vskip 0.3in
]



\printAffiliationsAndNotice{\icmlEqualContribution} % otherwise use the standard text.

\begin{abstract}

We study the certifiable robustness of ML classifiers on dirty datasets that could contain missing values.
A test point is \emph{certifiably robust} for an ML classifier if the classifier returns the same prediction for that test point, regardless of which cleaned version (among exponentially many) of the dirty dataset the classifier is trained on.
In this paper, we show theoretically that for Naive Bayes Classifiers (NBC) over dirty datasets with missing values: (i) there exists an efficient polynomial time algorithm to decide whether multiple input test points are all certifiably robust over a dirty dataset; and (ii) the data poisoning attack, which aims to make all input test points certifiably non-robust by inserting missing cells to the clean dataset, is in polynomial time for single test points but \textbf{NP}-complete for multiple test points. Extensive experiments demonstrate that our algorithms are efficient and outperform existing baselines.
%First, we present an efficient algorithm that \emph{decides} whether multiple test points are certifiably robust for NBC. Next, we present a novel poisoning attack that exploits missing values, aimed at making the test point not certifiably robust for NBC. Our theoretical analysis shows that the proposed attacks can \emph{poison} a clean dataset by inserting the optimal number of missing values such that a single test point is not certifiably robust for NBC. In addition, we also prove that poisoning a clean dataset such that multiple test points all become certifiably non-robust is \textbf{NP}-hard. The experiments show that our algorithm outperforms existing baselines.

% Our experiments demonstrate the efficiency and effectiveness of our decision algorithm and proposed attacks across ten real-world datasets. 

% achieve up to $19.5\times$ and $3.06\times$ speed-up over the baseline algorithms across ten real-world datasets. 

% Furthermore, we indentify several research directions that can guide researchers to integrate \emph{Certain Predictions} into real data cleaning systems based on the experimental results. 

% \zhiwei{here should we be more specific?(i.e., talking about speedups for decision and poisoning separately; talk about the practical inefficiency observed on counting)}
% \xiating{edited.}
\end{abstract}

\section{Introduction}

The reliability of Machine Learning (ML) applications is heavily contingent upon the integrity of the training data. However, real-world datasets are frequently plagued with issues such as missing values, noise, and inconsistencies~\cite{pwl+21, acd+16, cik+16, xbz+21, DBLP:conf/sigmod/PicadoDTL20}. Traditional approaches to handling such ``dirty'' datasets typically involve data cleaning.
%In this paper, we specifically focus on the problem of handling dirty datasets that contain missing values marked by \texttt{NULL}.
\begin{figure*}[!t]
    \centering
    \includegraphics[width=0.9\textwidth]{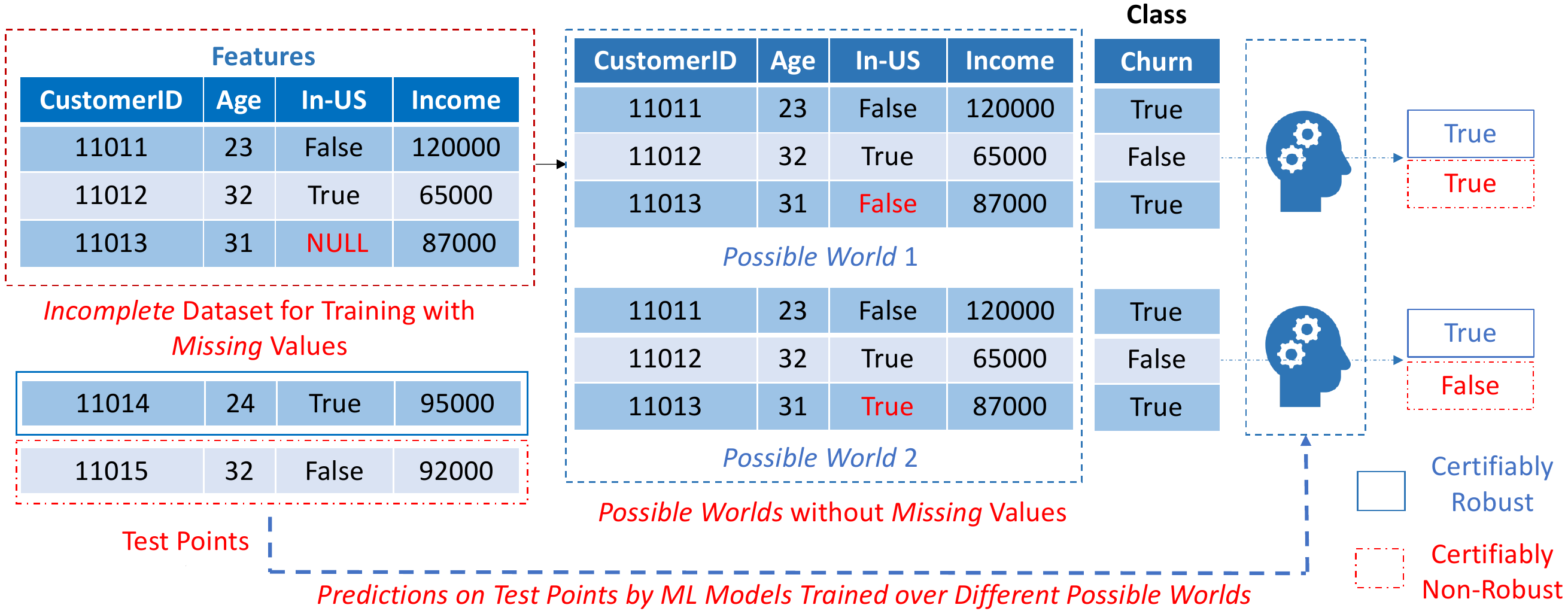}
    \caption{Example of an incomplete dataset, the possible worlds induced by the incomplete dataset, and the certifiably robust/non-robust predictions for data points with missing values over the possible worlds.
    % \xiating{From my understanding, we do not insist on a specific ML model used in this example. Please confirm (and if so, should we make one such that it is specific to NBC?)}
    }
    \label{fig:example}
    % \vspace{-6mm}
\end{figure*}
%Data cleaning is the most commonly used approach to address these issues, with the goal of producing high-quality data.
% A long line of research has focused on developing powerful data cleaning frameworks and tools~\cite{wkf+14,kww+16,wfg+17,rci+17}.
% using techniques from error detection~\cite{acd+16, mac+19}, missing value imputation~\cite{tkf+13}, and data deduplication~\cite{cik16}, etc. 
% Multiple data cleaning frameworks have been developed~\cite{}.
% Data cleaning has also been studied under different contexts \cite{kohler2021possibilistic,DBLP:journals/pvldb/ChengCX08,bertossi2013data,DBLP:conf/sigmod/KhayyatIJMOPQ0Y15,DBLP:conf/icde/BohannonFGJK07,DBLP:journals/pvldb/ProkoshynaSCMS15}. 
%However, data cleaning is commonly seen as a laborious and time-intensive process in data analysis, albeit various efforts to accelerate the data cleaning process~\cite{DBLP:journals/pvldb/RekatsinasCIR17,DBLP:conf/icde/ChuIP13,DBLP:conf/sigmod/ChuMIOP0Y15,DBLP:journals/pvldb/RezigOAEMS21}.
While effective, data cleaning can be an arduous and resource-intensive process, albeit various efforts to accelerate the data cleaning process~\cite{DBLP:conf/icde/ChuIP13,DBLP:conf/sigmod/ChuMIOP0Y15,DBLP:journals/pvldb/RekatsinasCIR17,DBLP:journals/pvldb/RezigOAEMS21}.% Moreover, current data cleaning frameworks often are agnostic to the downstream applications, leading to extraneous costs spent on the cleaning process and causing unnecessary delays~\cite{DBLP:conf/sigmod/WangKFGKM14,DBLP:journals/pvldb/AltowimKM14}. 
%In light of this, a natural question to ask is: 
This raises a pivotal question in the domain of ML: 
\begin{quote}
% \vspace{-2mm}
%{\em Can we reduce the cost of data cleaning by taking into account the downstream ML application}? 
{\em Can we circumvent the exhaustive data cleaning process by understanding the impact of data quality on ML models, particularly focusing on the notion of certifiable robustness}?
% \vspace{-2mm}
\end{quote}
We explain it with the following example.

\begin{example}
\label{ex:insurance}
Consider a data analyst working at an insurance company trying to predict whether an existing customer will leave (churn) using an ML classifier. 
Figure~\ref{fig:example} illustrates three customers with age, residency status (inside or outside the US), and income, which are features used for this classification task. 
The data for the customer with CustomerID 11013 is incomplete, as the residency information is missing (marked with \texttt{NULL}). Therefore there are two possible ways to clean the dataset: we can assign the missing feature ``In-US'' to either \textsf{True} or \textsf{False}, resulting in two possible worlds generated from the incomplete dataset. 

For the customer with CustomerID 11014, the ML classifier may predict \textsf{True} for that customer in both possible worlds. If this information is known before training and making predictions, data cleaning can be skipped since the prediction is {\em robust} with respect to the missing value.
However, the customer with CustomerID 11015 presents a different scenario, where the predictions made by the model trained on two possible worlds disagree: the ML classifier would predict \textsf{True} in the possible world where the feature ``In-US'' is \textsf{False}, but in the other possible world, the prediction is \textsf{False}.
Therefore, data cleaning is necessary to predict the customer with CustomerID 11015 accurately.
\end{example}

%Inspired by Example~\ref{ex:insurance}, a preliminary step is the study of \emph{certifiable robustness} for ML classifiers.
We investigate whether ML classifiers produce consistent predictions over incomplete (or dirty) data through the notion of \emph{certifiable robustness}.
A possible world (or cleaned dataset) of an incomplete dataset with missing values can be obtained by imputing each missing value marked as \texttt{NULL} with a valid value. 
Formally, a test point is said to be \emph{certifiably robust} for an ML classifier if its prediction for that test point remains the same regardless of which possible world (among exponentially many) generated from the incomplete dataset it was trained on. 
Otherwise, the test point is said to be \emph{certifiably non-robust}. This notion is helpful because if a test point is certifiably robust, then data cleaning is unnecessary because no matter how we clean the dataset to yield a possible world, the trained classifier will always give the same prediction.
% In Example~\ref{ex:insurance}, the customer with CustomerID 11014 is certifiably robust for the chosen ML classifier over the incomplete dataset, but the customer with CustomerID 11015 is certifiably non-robust.
Therefore, deciding whether this case holds efficiently would be beneficial to save the cost of data cleaning.
We formalize this into the following problem.

\smallskip
\noindent {\sf Decision Problem:} Are given test points certifiably robust for an ML classifier on a given incomplete dataset?

In addition, we are interested in how sensitive the prediction of an ML classifier over a complete dataset is to the presence of missing data. 
We study such sensitivity through the lens of an attack model, in which the attacker attempt to modify the least amount of original cell values to \texttt{NULL} such that the test point becomes certifiably non-robust.
This is formulated into the following problem.

\smallskip
\noindent {\sf Data Poisoning Problem:} Given a number of test points and a clean dataset, what is the minimum number of cells to be modified to \texttt{NULL} such that all test points become certifiably non-robust?
% This problem involves identifying the minimum number of cells in a clean dataset that need to be modified to \texttt{NULL} in order to make one or more test points certifiably non-robust for the ML classifier.
\smallskip

% Efficient algorithms for the \textsf{Decision Problem} can inform us of whether improving data quality is necessary~\cite{klw+20}.
%It is trivially true that any test point is certifiably robust for a given ML model on any complete dataset.
% Solving the \textsf{Data Poisoning Problem} can provide an overview of the impact of data quality on certifiable robustness.
% For example, given a single test point, if the \textsf{Data Poisoning Problem} yields that it is required to poison a minimum of $10\%$ cells from the clean data so that the given test point is certifiably non-robust, we know that as long as the number of dirty cells do not exceed $10\%$ of all cells, the test point will remain certifiably robust even in the precense of such dirty cells. Furthermore, the answer to the \textsf{Data Poisoning Problem} can show whether the ML classifier is sensitive to the data quality.
% To the best of our knowledge, there is no prior work that considers the Data Poisoning Problem for certifiable robustness of ML classifiers.

% how many cells we are required to clean to make the test points satisfy certifiable robustness.

% estimates on the safety of the training dataset from attacks and identify vulnerable datasets to poisoning, a task overlooked by prior work.

While verifying whether a test point is certifiably robust is in general a computationally expensive task since it may require examing exponentially many possible worlds, it is possible to do so efficiently for some commonly used ML models. 
% For example, recent research has demonstrated that it is possible to check whether a test point is certifiably robust for the $k$-Nearest Neighbor Classifier ($k$-NN) on datasets with missing values in polynomial time in the size of the dataset~\cite{klw+20, fk22}.
% Thus, it is natural to consider whether it is possible to extend certifiable robustness to other machine learning models. 
In this paper, we study the certifiable robustness of the Naive Bayes Classifier (NBC), a widely-used model in fields such as software defect prediction~\cite{vs17, aa17}, education~\cite{rss+17}, bioinformatics~\cite{acc+22, wan23, slz+19, yxf+19}, and cybersecurity~\cite{sdh+98, kns17}. Specifically, we focus on handling dirty datasets that contain missing values marked by \texttt{NULL}.

\smallskip
\noindent {\bf Contributions.}
In this paper, we study the two problems (i.e., Decision, and Data Poisoning) related to certifiable robustness for Naive Bayes Classifier (NBC), a simple but powerful supervised learning algorithm for predictive modeling that is widely used in academia and industry. 
Our main contributions are summarized as follows:

%Furthermore, we also consider the dual problem of Q1 here: given a complete dataset, a ML model, and a test point $x_{test}$, how many cells we need to attack to make the test point uncertain robust? 
% And the ML model in this paper is Naive Bayer Classifier.

%\noindent \textbf{Contributions.} Our main contributions are summarized as follows:
%\begin{itemize}
%    \item We develop the algorithms for answering Q1 and Q2. As for Q1, we state an algorithm to answer Q1 by going through the dataset for one time. And for Q2, we give an polynomial algorithm, and present the theoretial analysis to show that Q2 is NP-hard when the dimension of the dataset is not a constant. \Song{We did not present the hardness proof now.}
%    \item We propose the data poisoning problem based on certain predictions. And we present a greedy algorithm and the theoretical analysis shows that the algorithm is able to achieve the optimal solution.
%    \item We conduct extensive experiments over the four real datasets to compare our proposed algorithm with the baseline. The result indicates that our solution outperforms in speed and effectiveness.
%\end{itemize}
\begin{itemize}	
% \vspace{-3mm}
\item We show that the \textsf{Decision Problem} for NBC can be solved in time $O(md + nd)$, where $n$ is the number of data points in the dataset, $m$ is the number of labels, and $d$ is the number of features.
Since the training and classification for NBC also requires $O(md + nd)$ time, our algorithm exhibits no asymptotic overhead and provides a much stronger guarantee of the classification result compared to the original NBC;
% \vspace{-1mm}
% \item We present an algorithm for the \textsf{Counting Problem} for NBC that runs in time $O(dn^{md+1})$. While inefficient in practice, our algorithm does not explicitly enumerate every possible world of the incomplete dataset;  
% \zhiwei{We should mention the practical inefficiency and explain about it, emphasizing the theoretical contribution for this problem.}\xiating{edited.}

\item We show that for a single test point, the \textsf{Data Poisoning Problem} for NBC can be solved in time $O(nmd)$; and for multiple test points, the \textsf{Data Poisoning Problem} is \textbf{NP}-complete for datasets containing at least three features. We also provide an efficient heuristic algorithm for multiple points case.
% \vspace{-1mm}
\item We conduct extensive experiments using ten real-world datasets and demonstrate that our algorithms are far more efficient than baseline algorithms over \textsf{Decision Problem} and \textsf{Data Poisoning Problem}.\footnote{Our implementation is publicly available at \url{https://github.com/Waterpine/NBC-Missing}.}
% \vspace{-2mm}
% exhibit up to $19.5\times$ speed-up over the baseline algorithms for the \textsf{Decision Problem} and $3.06\times$ speed-up for the \textsf{Data Poisoning Problem}.
% \item Based on the experimental results, we identify several future research directions that can help us integrate the algorithms into data cleaning systems.
% \zhiwei{more specific about the speed-up.}\xiating{added.}
\end{itemize}

%\noindent \textbf{Organization.} The rest of this paper is organized as follows. %We introduce the preliminaris in Section~\ref{sec:preliminaries}. Section~\ref{sec:decision} presents the solutions to Q1. As for the Section~\ref{sec:counting}, we give the techniques for Q2. And we present the approaches for data poisoning in Section~\ref{sec:poisoning}. Section~\ref{sec:experiments} shows the experimental results. We discuss the related work in Section~\ref{sec:related}. Finally, we conclude this paper in Section~\ref{sec:conclusion}.
% \noindent \textbf{Organization.} The remainder of this paper is organized as follows. 
% Section~\ref{sec:preliminaries} introduces the notations used in this paper followed by necessary background and formal problem definition. 
% We present the algorithms to solve the \textsf{Decision Problem} for NBC in Section~\ref{sec:decision}. 
% In Section~\ref{sec:poisoning}, we study the \textsf{Data Poisoning Problem} for single test point and multiple test points.
% Section~\ref{sec:experiments} presents the experimental evaluation. 
% We discuss the related work in Section~\ref{sec:related}. Finally, we conclude our paper and discuss the future work in Section~\ref{sec:conclusion}.

%\xiating{The full version of this paper can be accessible at \texttt{a github link}. Since we only have the hardness proof in the appendix, we could also just mention this in the poisoning section instead of introduction.}

\section{Preliminaries}
\label{sec:preliminaries}
% \xiating{well, for ICML, maybe we can omit some?} \Song{I agree with your points.}

\noindent \textbf{Data Model.}
We assume a finite set of domain values for each attribute (or feature). All continuous features are properly discretized into bins/buckets. For simplicity, we assume that every feature shares the same domain ${U}$, which contains a special element \texttt{NULL}. 
The set $\mathcal{X} = {U}^d$ is a feature space of dimension $d$. 
A datapoint $\mathbf{x}$ is a vector in $\mathcal{X}$ and we denote $\mathbf{x}_i$ as the attribute value at the $i$-th position of $\mathbf{x}$. 
We assume a labeling function $\ell: \mathcal{X} \rightarrow \mathcal{Y}$ for a finite set of labels $\mathcal{Y}$. 
A training dataset $\mathcal{D}$ is a set of data points in $\mathcal{X}$, each associated with a label in $\mathcal{Y}$. 

\smallskip
\noindent \textbf{Missing Values and Possible Worlds.}
A data point $\mathbf{x}$ contains missing value if $\mathbf{x}_i = \texttt{NULL}$ for some $i$. 
A dataset $\mathcal{D}$ is complete if it does not contain data points with missing values, or is otherwise incomplete.
We denote an incomplete dataset as $\idata$ and assume that all data points have a label.
For an incomplete dataset $\idata$, a possible world can be obtained by replacing each attribute value marked with \texttt{NULL}  with a domain value that exists in $\idata$. This follows the so-called {\em closed-world semantics} of incomplete data. We denote $\mathcal{P}(\idata)$ as the set of all possible worlds that can be generated from $\idata$. 
%For simplicity, we assume that the domain $U$ is precisely the set of every domain value in $\idata$, and therefore the size of $U$ is at most the number of data points in $\idata$.
Given an ML classifier $f$, we denote $f_{\mathcal{D}}: \mathcal{X} \rightarrow \mathcal{Y}$ as the classifier trained on a complete dataset $\mathcal{D}$ that assigns for each complete datapoint $\test \in \mathcal{X}$, a label $f_{\mathcal{D}}(\test) = l \in \mathcal{Y}$.

\smallskip
\noindent \textbf{Certifiable Robustness.} 
Given an incomplete dataset $\idata$ and an ML classifier $f$, a test point $\test$ is {\em certifiably robust} for $f$ over $\idata$ if there exists a label $l \in \mathcal{Y}$ such that $f_{\mathcal{D}}(\test) = l$ for any possible world $\mathcal{D} \in \mathcal{P}(\idata)$. Otherwise, the test point $\test$ is said to be \emph{certifiably non-robust}.

% If a test point $\test$ is certifiably robust, data cleaning on the incomplete dataset $\idata$ is unnecessary, since regardless of which possible world (or in other words, a clean dataset) we use as the training dataset, the prediction of $\test$ will remain the same. 
% We remark that every test point is vacuously certifiably robust for any deterministic ML classifier over a complete dataset. 
% \smallskip
% \noindent \textbf{Missing Values and Possible Worlds.} In our setting, the training dataset is {\em incomplete} and contains missing values. In particular, if the value of the $i$-th feature of point $\bm{x}$ is missing, we set $x_i = \texttt{NULL}$. We will use $\idata$ to denote such an incomplete dataset.

% To interpret the semantics of missing values, we use the notion of possible worlds. A {\em possible world} $\mathcal{D}$ of an incomplete dataset $\idata$ is a dataset where every $\texttt{NULL}$ in the $i$-th feature is replaced by a value from the active domain of that feature, i.e., $\{x_i \mid \bm{x} \in \mathcal{D}^\square, x_i \neq \texttt{NULL} \}$. In other words, a possible world represents a way to clean the training dataset. We will use $\mathcal{P}(\mathcal{D}^\square)$ to denote the set of all possible worlds that are generated from $\mathcal{D}^\square$.

\smallskip
\noindent \textbf{Naive Bayes Classifier.} Naive Bayes Classifier (NBC) is a simple but widely-used ML algorithm. 
Given a complete dataset $\mathcal{D}$ and a complete datapoint $\test = (x_1, x_2, \dots, x_d)$ to be classified, the datapoint is assigned to the label $l \in \mathcal{Y}$ such that the probability $\prob[l \mid \test]$ is maximized. 
NBC assumes that all features of $\mathcal{X}$ are conditionally independent for each label $l \in \mathcal{Y}$ and estimates, by the Bayes' Theorem, that
\begin{align*}
\prob[l \mid \test] 
%&= \prob[l] \cdot \prob[\test]^{-1} \cdot \prob[x_1, x_2, \dots, x_d \mid l]  \\
&= \prob[l] \cdot \prob[\test]^{-1} \cdot \prob[\test \mid l]  \\
&=  \prob[l] \cdot \prob[\test]^{-1} \cdot \prod_{j=1}^{d} \prob[x_j \mid l].
\end{align*}
% Therefore,
% %
% \begin{align*}
%   \Pr(l \mid x_{1}, x_{2}, \dots, x_{d}) &~\propto \Pr(l, x_{1}, x_{2}, \dots, x_{d}) \\
%   &~\propto \Pr(l) \Pr(x_{1} \mid l) \Pr(x_{2} \mid y) \Pr(x_{d} \mid l) 
% %  &~\propto \Pr(y) \prod_{i=1}^{n} \Pr(x_{i}|y)
% \end{align*}
% where $\propto$ indicates proportionality.
Finally, it assigns for each test point $\test = (x_1, x_2, \dots, x_d)$, the label $l$ that maximizes $\prob[l \mid \test]$.

To estimate the probabilities $\prob[l]$ and $\prob[x_j \mid l]$, NBC uses the corresponding relative frequency in the complete dataset $\mathcal{D}$, which we denote as $\Pr(l)_{\mathcal{D}}$ and $\Pr(x_j \mid l)_{\mathcal{D}}$ respectively.
When a (complete or incomplete) dataset $\mathcal{D}$ of dimension $d$, a test point $\test = (x_1, x_2, \dots, x_d)$ and $m$ labels $l_1, l_2, \dots, l_m$ are understood or clear from context, 
we denote $N_i$ as the \underline{n}umber of data points in $\mathcal{D}$ with label $l_i$, $E_{i,j}$ as the number of \underline{e}xisting data points $\mathbf{x}$ in $\mathcal{D}$ with label $l_i$ and has value $x_j$ as its $j$-th attribute, and $M_{i,j}$ as the number of data points $\mathbf{x}$ in $\mathcal{D}$ with label $l_i$ and has \underline{m}issing value \texttt{NULL} as its $j$-th attribute.
Hence for a complete dataset $\mathcal{D}$ with $n$ datapoints, NBC would first estimates that 
\begin{align*}
\prob[l_i] &~\approx \Pr(l_i)_{\mathcal{D}} = N_i / n \\
\prob[x_j \mid l_i] &~\approx \Pr(x_j \mid l_i)_{\mathcal{D}} = E_{i,j} / N_i
\end{align*}
and then compute a value $\support{\mathcal{D}}{l}{\test}$, called a support value of $\test$ for label $l$ in $\mathcal{D}$ given by
\begin{align*}
\support{\mathcal{D}}{l}{\test} 
&= \Pr(l)_{\mathcal{D}} \cdot \Pr(\test \mid l)_{\mathcal{D}} \\
&= \Pr(l)_{\mathcal{D}} \cdot \prod_{1 \leq j \leq d} \Pr(x_j \mid l)_{\mathcal{D}}.
\end{align*}
Finally, it predicts that
\begin{align*}
  f_{\mathcal{D}} (\test) = \arg\max_{l \in \mathcal{Y}} \quad \support{\mathcal{D}}{l}{\test}
\end{align*}
Note that given a dataset $\mathcal{D}$ and a test point $\test$, the frequencies $N_i$, $E_{i,j}$ and $M_{i,j}$ can all be computed in time $O(nd)$.

For a possible world $\mathcal{D}$ generated from an incomplete dataset $\idata$ and a test point $\test = (x_1, x_2, \dots, x_d)$, we use $\alpha_{i,j}(\mathcal{D})$ to denote the number of data points in $\idata$ with label $l_i$ and value \texttt{NULL} at its $j$-th attribute in $\idata$ and \underline{a}ltered to $x_j$ in $\mathcal{D}$.
It is easy to see that $0 \leq \alpha_{i,j}(\mathcal{D}) \leq M_{i,j}.$

\smallskip
\noindent \textbf{Problem Definitions.} 
For the \textsf{Decision Problem} $\crnbc$, we are interested in deciding whether a test point is certifiably robust for NBC:

% \begin{definition}[Decision Problem]
%   \label{def:decision}
% Given an incomplete dataset $\mathcal{D}^\square$ and a test point $\test$, is $\test$ certifiably robust under the Naive Bayes classifier trained over $\mathcal{D}^\square$? 
% \end{definition}

\begin{quote}
Given an incomplete dataset $\idata$ and a test point $\test$, is $\test$ certifiably robust for NBC over $\idata$?
\end{quote}

As will become apparent in our technical treatments later, we can easily extend our algorithm for $\crnbc$ on a single test point to multiple test points.

Conversely, an adversary may attempt to attack a complete dataset by inserting missing values, or replacing (poisoning) existing values with new ones. 
The data poisoning problem $\poisoncrnbc$ asks for the fewest number of attacks that can make all given test points certifiably non-robust:

% \begin{definition}[Data Poisoning Problem]
%   \label{def:poisoning}
% Given a complete dataset $\mathcal{D}$ and a test point $\test$, find a poisoned instance $\mathcal{D}^\dagger$ of $\mathcal{D}$  that has as few missing values as possible such that the prediction on $\test$ by Naive Bayes classifier trained over $\mathcal{D}^\dagger$ is certifiably non-robust.
% \end{definition}

\begin{quote}
Given a complete dataset $\mathcal{D}$ and test points $\test_1$, $\test_2$, $\dots$, $\test_k$, find a poisoned instance $\mathcal{D}^\dagger$ of $\mathcal{D}$ that has as few missing values as possible such that every $\test_i$ is certifiably non-robust for NBC on $\mathcal{D}^\dagger$.
\end{quote}

% \begin{description}
% \item[Problem:] $\certrobustpoison{f}$ 
% \item[Input:] a complete dataset $\mathcal{D}$ and test points $\test_1$, $\test_2$, $\dots$, $\test_k$
% \item[Question:] find a poisoned instance $\mathcal{D}^\dagger$ of $\mathcal{D}$ that has as few missing values as possible such that every $\test_i$ is certifiably non-robust for $f$ on $\mathcal{D}^\dagger$.
% \end{description}

A solution to $\poisoncrnbc$ provides us with a notion of the robustness of the training dataset against poisoning attacks. For example, if the ``smallest'' poisoned instance has $100$ missing values, this means that an attacker needs to change at least $100$ values of the dataset to alter the prediction of the targeted test point(s). In addition, it also implies that as long as fewer than $100$ cells are poisoned, the prediction of the test point will not be altered even in the presence of missing data.

% \newpage
% \clearpage

\section{Decision Algorithms}
\label{sec:decision}
% \xiating{In the original manuscript, we have one introductory algorithm and one optimized algorithm. This current version on 2/16/2023 only presents the optimized algorithm for simplicity reasons, but there are also reasons to present the introductory algorithm: (1) it is easier for the readers, (2) it takes up more space in the paper (we are currently at 10 pages), and (3) our experiment would have one more baseline algorithm.   
% }

In this section, we give an efficient algorithm for $\crnbc$. Let $\idata$ be an incomplete dataset and let $\test$ be a test point.
For an arbitrary label $l$, consider the maximum and minimum support value of $\test$ for $l$ over all possible world $\mathcal{D}$ of $\idata$. 
We define that 
\begin{align*}
    \maxsupport{\idata}{l}{\test} &~:= \max_{\mathcal{D} \in \mathcal{P}(\idata)} \support{\mathcal{D}}{l}{\test} \\
    \minsupport{\idata}{l}{\test} &~:= \min_{\mathcal{D} \in \mathcal{P}(\idata)} \support{\mathcal{D}}{l}{\test}
\end{align*}

Our algorithm relies on the following observation, which was first studied in~\cite{ramoni2001robust}.

\begin{lemma}
\label{lemma:decision}
A test point $\test$ is certifiably robust for NBC over $\idata$ if and only if there is a label $l$ such that for any label $l' \neq l$, $\minsupport{\idata}{l}{\test} > \maxsupport{\idata}{l'}{\test}.$
\end{lemma}
The full proof of Lemma~\ref{lemma:decision} is given in Appendix~\ref{sec:decision-proof-appendix}.

% \begin{quote}
% A test point $\test$ is certifiably robust for NBC over $\idata$ if and only if there is a label $l$ such that for any label $l' \neq l$, $\minsupport{\idata}{l}{\test} > \maxsupport{\idata}{l'}{\test}.$
% \end{quote}

\begin{algorithm}[!ht]
    \caption{Iterative Algorithm for the Decision Problem}\label{alg:improved_decision_miss}
    \KwIn{Incomplete dataset $\idata$, test point $\test$}
    \KwOut{Is $\test$ certifiably robust?}
%$m \leftarrow$ the number of labels in dataset $\mathcal{D}$, $d \leftarrow$ the number of dimension of the data point in $\mathcal{D}$ \\
    % $\mathcal{T}[m][d] \leftarrow [][]$, $\mathcal{E}[m][d] \leftarrow [][]$, $\mathcal{M}[m][d] \leftarrow [][]$, $\mathcal{L}[m] \leftarrow []$, $\mathcal{U}[m] \leftarrow []$ \\
    \ForEach{label $l_i \in \mathcal{Y}$}{
        $N_i \leftarrow$ the number of data points in $\idata$ with label $l_i$ \\
        \For{$j=1$ to $d$}{
            $E_{i,j} \leftarrow$ the number of existing data points $\bm{x}$ in $\idata$ with label $l_i$ and $\bm{x}_j = \test_j$. \\
            $M_{i,j} \leftarrow$ the number of missing data points $\bm{x}$ in $\idata$ with label $l_i$ and $\bm{x}_j = \texttt{NULL}$. \\
            % $U_{i,j} \leftarrow \text{E}_{i,j} + \text{M}_{i,j}$ \\
        }
    }
    \ForEach{label $l_i \in \mathcal{Y}$}{
        $S_{i}^{\downarrow}\leftarrow N_i^{-(d-1)} \cdot \prod_{1 \leq j \leq d} E_{i,j}$ \\
        $S_{i}^{\uparrow} \leftarrow N_i^{-(d-1)} \cdot \prod_{1 \leq j \leq d} (E_{i,j}+M_{i,j})$ \\
    }
    \If{$\exists l_i \in \mathcal{Y}$ such that $S_{i}^{\downarrow} > \max_{j \neq i} S_{j}^{\uparrow}$} {
        \Return true \\
    }
    \Return false
\end{algorithm}

For each label $l_i$, both quantities $\minsupport{\idata}{l_i}{\test}$ and $\maxsupport{\idata}{l_i}{\test}$ can be computed efficiently.
Intuitively, it suffices to inspect only the ``extreme'' possible world that is the worst and best for $l_i$, respectively.
We can achieve this by assigning all the missing cells to disagree or agree with the test point on the corresponding attributes. 

We now formally prove this fact. Fix a possible world $\mathcal{D}$ of $\idata$, and let $E_{i,j}$, $\alpha_{i,j}(\mathcal{D})$ and $N_i$ be relative to $\idata$ and $\test$. We have
\begin{align*}
\support{\mathcal{D}}{l_i}{\test}
&= \Pr(l_i)_{\mathcal{D}} \cdot \prod_{1 \leq j \leq d} \Pr(x_j \mid l_i)_{\mathcal{D}} \\
&= \frac{N_i}{n} \cdot \prod_{1 \leq j \leq d} \frac{E_{i,j} + \alpha_{i,j}(\mathcal{D})}{N_i},
\end{align*}
which is maximized when every $\alpha_{i,j}(\mathcal{D}) = M_{i,j}$ and minimized when every $\alpha_{i,j}(\mathcal{D}) = 0$. 
Moreover, both extreme values are attainable. 
Hence 
\begin{equation}
\label{eq:upperbound}
\maxsupport{\idata}{l_i}{\test} = \frac{N_i}{n} \cdot \prod_{1 \leq j \leq d} \frac{E_{i,j} + M_{i,j}}{N_i}
\end{equation}
\begin{equation}
\label{eq:lowerbound}
\minsupport{\idata}{l_i}{\test} = \frac{N_i}{n} \cdot \prod_{1 \leq j \leq d} \frac{E_{i,j}}{N_i}.
\end{equation}

Our algorithm is presented in Algorithm~\ref{alg:improved_decision_miss}. We first compute, for each label $l$, Eq.~(\ref{eq:upperbound}) and~(\ref{eq:lowerbound}) in line 1--9, and then check whether there is some label $l_i$ such that for any label $l_j \neq l_i$, 
\begin{equation}
\minsupport{\idata}{l_i}{\test} > \maxsupport{\idata}{l_j}{\test},
\end{equation}
where a strict inequality is required to return true.

\smallskip
\noindent \textbf{Running Time.} 
Line $1$--$5$ runs in time $O(nd)$, where $n$ is the number of points in $\mathcal{D}$, and $d$ is the number of dimension of points in the dataset. Line $6$--$8$ takes $O(md)$ time, where $m$ is the number of labels in the dataset. 
Line $9$--$11$ can be implemented in $O(m)$, by first computing each $\max_{j \neq i} S_j^{\uparrow}$ in $O(m)$ and then iterating through every $l_i$. In conclusion, the total time complexity of Algorithm~\ref{alg:improved_decision_miss} is $O(nd)$.

\smallskip
\noindent \textbf{Extension to Multiple Test Points.} 
Given $k$ test points $\test_1, \test_2, \dots, \test_k$, a trivial yet inefficient way is to run Algorithm~\ref{alg:improved_decision_miss} for every test point, giving a running time of $O(knd)$.
We show that Algorithm~\ref{alg:improved_decision_miss} can be easily adapted to check whether multiple test points are all certifiably robust efficiently with the help of an index in time $O(nd + kmd)$:
We can modify line $1$--$5$ to compute an index $E_{i,j}[x_j]$, which represents the number of existing data points $\bm{x}$ in $\idata$ with label $l_i$ and value $\mathbf{x}_j$ in $\idata$, instead of $\test_j$ in the original algorithm. This runs in $O(nd)$ time and requires $O(md)$ space. Then we iterate line $6$--$11$ for each of the $k$ test points, where in each iteration, both quantities in line $7$--$8$ can be efficiently computed by checking the index $E_{i,j}[x_j]$ in $O(md)$ time. The overall running time is thus $O(nd + kmd)$.

\section{Data Poisoning Algorithms}
\label{sec:poisoning}

% The previous sections present algorithms for deciding whether a test point is certifiably robust for NBC when the dataset could contain missing cells. 
In this section, we consider the setting where an attacker attempts to poison a complete dataset by injecting some missing cells so that a given set of test points becomes certifiably non-robust. From this setting, we can better comprehend if certifiable robustness is a stringent requirement within specific models. 
We first present a simple greedy algorithm that solves $\poisoncrnbc$ optimally for a single test point in Section~\ref{sec:single}.
We then show that $\poisoncrnbc$ is \textbf{NP}-complete for multiple test points and provide an efficient heuristic algorithm in Section~\ref{sec:multiple}.

% Formally, we consider the following data poisoning problem for Naive Bayes Classifier:

% \begin{description}
% 	\item[Problem] \textsc{DataPoisoning}
% 	\item[Input] a complete dataset $\mathcal{D}$, a test point $\test$
% 	\item[Output] an incomplete dataset $\mathcal{D}^{\dagger}$, obtained by setting the minimum number of cells in $\mathcal{D}$ to \texttt{NULL} so that the test point $\test$ is not certifiably robust in $\mathcal{D}^{\dagger}$ for the Naive Bayes Classifier 
% \end{description}

\subsection{A Single Test Point}
\label{sec:single}

Let us first consider a simpler problem that does not involve missing values and considers only a single test point. 

\begin{definition}[AlterPrediction]
Given a complete dataset $\mathcal{D}$, a test point $\test$ with $l^* = f_{\mathcal{D}}(\test)$, a label $l \neq l^*$, the \textsf{AlterPrediction} problem is to find a complete dataset $\mathcal{D}'$, obtained by altering the minimum number of cells in $\mathcal{D}$ so that $\support{\mathcal{D}'}{l}{\test} > \support{\mathcal{D}'}{l^*}{\test}.$
\end{definition}

% \begin{description}
% 	\item[Problem] \textsf{AlterPrediction}
% 	\item[Input] a complete dataset $\mathcal{D}$, a test point $\test$ with $l^* = f_{\mathcal{D}}(\test)$, a label $l \neq l^*$
% 	\item[Output] a complete dataset $\mathcal{D}'$, obtained by altering the minimum number of cells in $\mathcal{D}$ so that 
% 	$$\support{\mathcal{D}'}{l}{\test} > \support{\mathcal{D}'}{l^*}{\test}.$$
% \end{description}

We show that an algorithm for \textsf{AlterPrediction} immediately leads to an algorithm for $\poisoncrnbc$.

% It is easy to see that if $\mathcal{D}'$ is a solution to the \textsf{AlterPrediction} problem, then $f_{\mathcal{\mathcal{D}'}}(\test) \neq l^*$. 
% Therefore, the the test point $\test$ is certifiably non-robust for NBC over the incomplete dataset $\idata$, obtained by setting the altered cells in $\mathcal{D}'$ to \texttt{NULL}, because both $\mathcal{D}$ and $\mathcal{D}'$ are possible worlds of $\idata$ and $f_{\mathcal{D}}(\test) = l^* \neq f_{\mathcal{D}'}(\test)$.
% For the other direction, if $\idata$ is a solution to the $\poisoncrnbc$ problem, then there is a possible world $\mathcal{D}'$ of $\idata$ such that $l = f_{\mathcal{D}'}(\test) \neq l^* = f_{\mathcal{D}}(\test)$, that is,
% $$S(l \mid \test)_{\mathcal{D}'} > S(l^* \mid \test)_{\mathcal{D}'}.$$
% Hence $\mathcal{D}'$ must also be a solution to \textsf{AlterPrediction}, with input $\mathcal{D}$, $\test$ and $l = f_{\mathcal{D}'}(\test)$.

\begin{lemma}
\label{lemma:poisoning-complete}
Let $\mathcal{D}$ be a complete dataset and $\test$ a test point. Let $l^* = f_{\mathcal{D}}(\test)$. Let $k$ be an integer. Then the following statements are equivalent:
\begin{enumerate}
	% \vspace{-3mm}
	\item There exists a solution $\mathcal{D}^{\dagger}$ for $\poisoncrnbc$ for a complete dataset $\mathcal{D}$ and a test point $\test$ with $k$ altered cells.
	% \vspace{-3mm}
	\item There exists a label $l\neq l^*$ and a solution $\mathcal{D}'$ for \textsf{AlterPrediction} for a complete dataset $\mathcal{D}$, a test point $\test$ and label $l^*$ with $k$ altered cells. 
	% \vspace{-3mm}
\end{enumerate}
\end{lemma}
The full proof of Lemma~\ref{lemma:poisoning-complete} is in Appendix~\ref{sec:poisoning-complete-appendix}

The outline of our algorithm for $\poisoncrnbc$ on a single test point is presented in Algorithm~\ref{alg:single-poisoning}. 
It thus remains to solve \textsf{AlterPrediction} efficiently.
% In the remainder of this subsection, we present a simple greedy algorithm that can solve \textsf{AlterPrediction} optimally. 

\begin{algorithm}[!ht]
\caption{$\poisoncrnbc$-Single}
\label{alg:single-poisoning}
\KwIn{A complete dataset $\mathcal{D}$, a test point $\test$}
\KwOut{An incomplete dataset $\mathcal{D}^{\dagger}$ obtained by setting the minimum number of cells in $\mathcal{D}$ to \texttt{NULL} such that $\test$ is not certifiably-robust for NBC }
$l^* = f_{\mathcal{D}}(\test)$\\
$\mathsf{minAlter} \leftarrow \infty$, $\mathcal{D}^{\dagger} \leftarrow \emptyset$ \\
\ForEach{$l$ in $\mathcal{Y} \setminus \{l^*\}$}{
    $\mathcal{D}_l \leftarrow \textsf{AlterPrediction}(\mathcal{D}, \test, l)$ \\
    $\mathsf{alter} \leftarrow$ the number of altered cells in $\mathcal{D}_l$ w.r.t.\ $\mathcal{D}$ \\
    \If{$\mathsf{alter} < \mathsf{minAlter}$}{
            $\mathsf{minAlter} \leftarrow \mathsf{alter}$ \\
            $\mathcal{D}^{\dagger} \leftarrow$ set all altered cells in $\mathcal{D}_l$ to \texttt{NULL}
    }
    
}
\Return $\mathcal{D}^{\dagger}$
\end{algorithm}

\smallskip
\noindent\textbf{An Optimal Greedy Algorithm for \textsf{AlterPrediction}}

Our \textsf{AlterPrediction} algorithm is presented in Algorithm~\ref{alg:alter-prediction}.
At a high level, it reduces the objective quantity $\Delta$ (at line $2$), by choosing (at lines $17$-$20$) the better operation of the two, one considered in lines $3$-$9$, and the other in lines $10$-$16$.

\noindent\textbf{Step 1.} Let us explain first on why the quantity $\Delta$ is crucial to the correctness of this algorithm.

Let $\mathcal{D}$, $\test$, $l^* = f_{\mathcal{D}}(\test)$ and $l\neq l^*$ be input to the \textsf{AlterPrediction} problem, where $\mathcal{D}$ is an incomplete dataset and $\test$ is a test point.

% \subsubsection{\textsf{AlterPrediction} as a Minimization Problem.}
We consider an iterative process where at each step, only one single cell is altered: We start from $\mathcal{D}_0 = \mathcal{D}$ and alter one single cell in $\mathcal{D}_0$ and obtain $\mathcal{D}_1$. We then repeat the same procedure on $\mathcal{D}_1$, until we stop at some $\mathcal{D}_k = \mathcal{D}'$.
Fix a label $l \neq l^*$. In the original dataset $\mathcal{D}$, the test point $\test = (x_1, x_2, \dots, x_n)$ is predicted $l^*$, and thus we have necessarily
$$\support{\mathcal{D}}{l^*}{\test} > \support{\mathcal{D}}{l}{\test},$$
or equivalently
$$ \Pr(l^*)_{\mathcal{D}_0} \cdot \Pr(\test \mid l^*)_{\mathcal{D}_0} - \Pr(\test \mid l)_{\mathcal{D}_0} \cdot \Pr(l)_{\mathcal{D}_0} > 0$$
If the ending dataset $\mathcal{D}' = \mathcal{D}_k$ is a solution to \textsf{AlterPrediction}, we have
$$\support{\mathcal{D}'}{l^*}{\test} < \support{\mathcal{D}'}{l}{\test}$$
or equivalently,
$$ \Pr(l^*)_{\mathcal{D}_k} \cdot \Pr(\test \mid l^*)_{\mathcal{D}_k} - \Pr(\test \mid l)_{\mathcal{D}_k} \cdot \Pr(l)_{\mathcal{D}_k} < 0.$$

We define the quantity 
$$\Delta_i := \Pr(\test \mid l^*)_{\mathcal{D}_i} \cdot \Pr(l^*)_{\mathcal{D}_i} - \Pr(\test \mid l)_{\mathcal{D}_i} \cdot \Pr(l)_{\mathcal{D}_i}$$
It is easy to see that the \textsf{AlterPrediction} problem is equivalent to finding a smallest $k$ and a sequence $\{\mathcal{D}_i\}_{0 \leq i \leq k}$ such that $\Delta_i > 0$ for $0 \leq i < k$ and $\Delta_k < 0$, which can be solved by finding for a fixed $k$, the dataset $\mathcal{D}_k$ with the smallest possible $\Delta_k$. 

\noindent\textbf{Step 2.} We now show that $\Delta$ can be reduced to negative using the fewest number of steps by performing only A1 or only A2 on the dataset $\mathcal{D}_i$ iteratively, where

\begin{enumerate}
\item[A1:] Alter datapoints with label $l$ in $\mathcal{D}_i$ so that $ \Pr(\test \mid l)_{\mathcal{D}_i}$ increases the most; and 
\item[A2:] Alter datapoints with label $l^*$ in $\mathcal{D}_i$ so that $ \Pr(\test \mid l^*)_{\mathcal{D}_i}$ decreases the most. 
\end{enumerate}

An example is illustrated in Example~\ref{ex:strategy}.

Note that $\Pr(l)_{\mathcal{D}_i}$ remains the same for any label $l$. 
Indeed, altering cells in points with labels not in $\{l, l^*\}$ does not change the value of neither $ \Pr(l\mid \test)_{\mathcal{D}_i}$ nor $ \Pr(l^* \mid \test)_{\mathcal{D}_i}$.
Besides, altering cells so that either $ \Pr(l^* \mid \test)_{\mathcal{D}_i}$ increases or $ \Pr(l \mid \test)_{\mathcal{D}_i}$ decreases is not helpful, because we would waste one alternation by increasing $\Delta_i$, not decreasing it.

Note that for any label $l$ and a test point $\test=(x_1, x_2, \dots, x_d)$, 
\begin{equation}
\label{eq:product}
\Pr(\test \mid l)_{\mathcal{D}_i} = \prod_{1 \leq j \leq d} \Pr(x_j \mid l)_{\mathcal{D}_i} = \prod_{1 \leq j \leq d} \frac{E_j}{N},
\end{equation}
where $N$ is the number of data points in $\mathcal{D}_i$ with label $l$ and $E_j$ is the number of data points in $N$ with label $l$ whose $j$-th attribute agrees with $x_j$. 
Note that $N$ is fixed, hence to make the largest increase or decrease in A1 or A2, we would always alter datapoints in $\mathcal{D}_i$ with the smallest nominator $E_j$ in $ \Pr(x_j\mid l )_{\mathcal{D}_i}$ so that the number of data points in $\mathcal{D}_{i+1}$ with label $l$ value $x_j$ as $j$-th attribute increases or decreases by $1$.

The following observations behind the strategies are crucial, which we formally prove in Section~\ref{sec:observation-args} in Appendix~\ref{sec:poisoning-appendix}. An example is provided in Example~\ref{ex:same-step}.
\begin{enumerate}
\item[O1:] Whenever we apply A1 to decrease $\Delta_i$, the reduction in $\Delta_i$ is non-decreasing in the step $i$.
\item[O2:] Whenever we apply A2 to decrease $\Delta_i$, the reduction in $\Delta_i$ remains the same across every step $i$.
\item[O3:] The reduction in $\Delta_i$ obtained by applying A1 (or A2) only depends on the number of times A1 (or A2) has been applied to obtain $\mathcal{D}_i$ previously. 
\end{enumerate}

Our key result is summarized by the following Lemma, which we prove in Section~\ref{sec:proof-key-lemma} in Appendix~\ref{sec:poisoning-appendix}.

\begin{lemma}
\label{lemma:key}
Let $\delta^+_{i}$ be the decrease in $\Delta_j$ whenever $\mathcal{D}_{j}$ is obtained by applying A1 in the sequence for the $i$-th time, implied by O1 and O3.
Let $\delta^-$ be the fixed decrease in $\Delta_i$ whenever we apply A2 to $\mathcal{D}_i$, implied by O2 and O3. 
Then 
\begin{equation}
\label{eq:goal}
\Delta_k \geq \Delta_0 - \max\{k\cdot \delta^-, \sum_{1 \leq j \leq k} \delta^+_j\},
\end{equation}
and equality is attainable.
\end{lemma}

Lemma~\ref{lemma:key} means that among all possible datasets $\mathcal{D}_k$ obtained by $k$ alternations to $\mathcal{D}$, the dataset with the smallest $\Delta_k$ can be obtained by either applying only A1 or only A2.

% \xiating{The intuition is that, the increase in $Pr(\test \mid l)_{\mathcal{D}}$ would grow when (1) is repeatedly applied, but the decrease in $Pr(\test \mid l^*)_{\mathcal{D}}$ is fixed.}
% We provide the formal arguments in Appendix~\ref{sec:greedy-is-correct}.

Algorithm~\ref{alg:alter-prediction} implements this idea, in which line 3--9 computes the poisoned dataset $\mathcal{D}^+$ with minimum number $k^+$ of alternations to $\mathcal{D}$ by applying A1. Similarly, line 10--16 computes $\mathcal{D}^-$ and $k^-$ respectively for repeatedly applying A2 to $\mathcal{D}$.

\begin{algorithm}[!ht]
\caption{\textsf{AlterPrediction}}
\label{alg:alter-prediction}
\KwIn{A complete dataset $\mathcal{D}$, test point $\test$, and a label $l \neq f_{\mathcal{D}}(\test)$}
\KwOut{A complete dataset $\mathcal{D}'$ with minimum number of altered cells in $\mathcal{D}$ such that $S(l^* \mid \test)_{\mathcal{D}'} < S(l \mid \test)_{\mathcal{D}'}$}
$l^* \leftarrow f_{\mathcal{D}}(\test)$ \\
% $\mathcal{D}' = \mathcal{D}$ \\
$\Delta = \Pr(\test \mid l^*)_{\mathcal{D}} \cdot \Pr(l^*)_{\mathcal{D}} - \Pr(\test \mid l)_{\mathcal{D}} \cdot \Pr(l)_{\mathcal{D}}$ \\
$\Delta^+ = \Delta$, $k^+ = 0$, $\mathcal{D}^+ = \mathcal{D}$ \\
\While{$\Delta^+ > 0$}{
	$j \leftarrow$ an attribute index with the smallest nominator in $\Pr(\test_j \mid l)_{\mathcal{D^+}}$\\
	$p^{+} \leftarrow$ a point in $\mathcal{D}^+$ with label $l$ and $p^+_{j} \neq x_{j}$ \\
	$\mathcal{D}^{+} \leftarrow $ set $p^+_{j}$ as $x_{j}$ \\
	$\Delta^{+} = \Pr(\test \mid l^*)_{\mathcal{D}^+} \cdot \Pr(l^*)_{\mathcal{D}^+} - \Pr(\test \mid l)_{\mathcal{D}^+} \cdot \Pr(l)_{\mathcal{D}^+}$ \\
	$k^+ \leftarrow k^+ + 1$\\
}
$\Delta^- = \Delta$, $k^- = 0$, $\mathcal{D}^- = \mathcal{D}$ \\
\While{$\Delta^- > 0$}{
	$j \leftarrow $ an attribute index with the smallest nominator in $\Pr(\test_j \mid l^*)_{\mathcal{D^-}}$ \\
	$p^{-} \leftarrow$ a point in $\mathcal{D}^-$ with label $l^*$ and $p^-_{j} = x_{j}$ \\
	$\mathcal{D}^{-} \leftarrow $ set $p^-_{j}$ as $c$ where $c \neq x_{j}$ \\
	$\Delta^{-} = \Pr(\test \mid l^*)_{\mathcal{D}^-} \cdot \Pr(l^*)_{\mathcal{D}^-} - \Pr(\test \mid l)_{\mathcal{D}^-} \cdot \Pr(l)_{\mathcal{D}^-}$ \\
	$k^- \leftarrow k^- + 1$\\
}
\If{$k^- < k^+$}{
	\Return $\mathcal{D}^-$
}
\Else{
	\Return $\mathcal{D}^+$
}

% \While{$\Delta > 0$}{
% 	$A_{i} \leftarrow$ the attribute with the smallest $N(\test_{i} \land l)$ in $\mathcal{D}'$ \\
% 	$x^{+} \leftarrow$ a point in $\mathcal{D}'$ with label $l$ and $x^+_{i} \neq \test_{i}$ \\
% 	$\mathcal{D}^{+} \leftarrow $ set $x^+_{i}$ as $\test_{i}$ \\
% 	$\Delta^{+} = Pr(\test \mid l^*)_{\mathcal{D}^+} \cdot Pr(l^*)_{\mathcal{D}^+} - Pr(\test \mid l)_{\mathcal{D}^+} \cdot Pr(l)_{\mathcal{D}^+}$ \\
% 	$A_{i^*} \leftarrow $ the attribute with the smallest $N(\test_{i^*} \land l^*)$ in $\mathcal{D}'$ \\
% 	$x^{-} \leftarrow$ a point in $\mathcal{D}'$ with label $l^*$ and $x^-_{i^*} = \test_{i^*}$ \\
% 	$\mathcal{D}^{-} \leftarrow $ set $x^-_{i^*}$ as $c$ where $c \neq \test_{i^*}$ \\
% 	$\Delta^{-} = Pr(\test \mid l^*)_{\mathcal{D}^-} \cdot Pr(l^*)_{\mathcal{D}^-} - Pr(\test \mid l)_{\mathcal{D}^-} \cdot Pr(l)_{\mathcal{D}^-}$ \\
	
% 	\If{$\Delta^- < \Delta^+$}{
% 		$\Delta = \Delta^-$ \\
% 		$\mathcal{D}' \leftarrow \mathcal{D}^-$
% 	}
% 	\Else{
% 		$\Delta = \Delta^+$ \\
% 		$\mathcal{D}' \leftarrow \mathcal{D}^+$
% 	}
	
% }
% \Return $\mathcal{D}'$
\end{algorithm}
% \vspace{-7mm}

\noindent \textbf{Running Time.} 
For Algorithm~\ref{alg:alter-prediction}, each \texttt{while} loop will be executed at most $n$ times. 
Lines $5$--$9$ and $12$--$16$ can be completed in constant time by preprocessing the table while executing line 1.
Hence Algorithm~\ref{alg:alter-prediction} runs in $O(nd + md)$ time. 
Algorithm~\ref{alg:single-poisoning} then takes $O(nmd)$ time.

% The time complexity of line $3-7$ is $O(nd)$, where $n$ is the number of points in dataset $\mathcal{D}$ and $d$ is the number of dimensions of the point in dataset $\mathcal{D}$. In addition, the time complexity of line $10$-$14$ is $O(m \cdot \mathcal{M})$, where $m$ is the cardinality of the label set and $\mathcal{M}$ is the maximum step computed in line $12$ to use only one strategy. Therefore, the total time complexity is $O(nd + m \cdot \mathcal{M})$.

\subsection{Multiple Test Points}
\label{sec:multiple}

We now consider the case in which an attacker wishes to attack a complete dataset $\mathcal{D}$ with minimum number of cells such that all $k$ test points $\test_1, \test_2, \dots, \test_k$ are certifiably non-robust for NBC.

A simple heuristic algorithm simply iteratively runs Algorithm~\ref{alg:single-poisoning} over every test point $\test_i$ and then take the union of all the missing cells found, which we present in Algorithm~\ref{alg:multiple-naive} in Appendix~\ref{sec:heuristic-algo-appendix} in detail.
While it can be easily verified that all test points are not certifiably robust for NBC over the incomplete dataset produced by Algorithm~\ref{alg:multiple-naive}, the number of poisoned cells is not necessarily minimal.
Indeed, compared with the single test point setting, the challenge of poisoning a dataset for multiple test points is that altering one single cell may affect the prediction of all test points that agree on that cell.

This observation allows us to show that the data poisoning problem for multiple test points is \textbf{NP}-complete for datasets over at least $3$ dimensions.
% It is then not surprising that for large dimension $d$, the problem \textsc{$d$-BatchPoisoning} is $\NP$-hard.

\begin{theorem}
\label{thm:batch-hard}
For every $d \geq 3$, $\poisoncrnbc$ is $\NP$-complete on datasets with $d$ dimensions and multiple test points.
\end{theorem}
The full proof of Theorem~\ref{thm:batch-hard} is deferred to Appendix~\ref{sec:batch-hard-appendix}, and an example reduction is provided in Example~\ref{ex:graph}.

\section{Experiments}
\label{sec:experiments}

In this section, we present the results of our experimental evaluation. we perform experiments on ten real-world datasets from Kaggle~\cite{web:kaggle}, and compare the performance of our efficient algorithms (as presented in Sections 3 and 5) against 
other straightforward solutions as baselines. 
\subsection{Experimental Setup}
We briefly describe here the setup for our experiments. 
\\
\noindent \textbf{System Configuration.} Our experiments were performed on a bare-metal server provided by Cloudlab~\cite{cloudlab}. The server is equipped with two 10-core Intel Xeon E5-2660 CPUs running at 2.60 GHz.
% With 157GB of memory, all experiments were performed entirely in memory, ensuring that performance was not affected by disk I/O. All algorithms were implemented and run as single-threaded programs, allowing us to focus on algorithmic performance without interference from parallelization or multithreading effects.

\noindent \textbf{Datasets.} We use ten real-world datasets from Kaggle: \textsf{heart (HE)}, \textsf{fitness-club (FC)}, \textsf{fetal-health (FH)}, \textsf{employee (EM)}, \textsf{winequalityN (WQ)}, \textsf{company-bankruptcy (CB)}, \textsf{Mushroom (MR)}, \textsf{bodyPerformance (BP)}, \textsf{star-classification (SC)}, \textsf{creditcard (CC)}. The details and the metadata of our datasets are summarized in Appendix~\ref{sec:datasets-appendix}. We first preprocess every dataset so that it contains only categorical features by partitioning each numerical feature into $5$ segments (or bins) of equal size using sklearn's KBinsDiscretizer~\cite{web:sklearn}.  %, where $n$  the number of bins equals 5.
The datasets are originally complete and do not have missing values. 

To obtain perturbed incomplete datasets with missing values from each of the ten datasets, we sample data cells to be marked as \texttt{NULL} as missing values uniformly across all cells from all features.
% in two ways: (i) Uniform Perturbation: uniformly across all cells from all features, and (ii) Feature Perturbation: uniformly across all cells from $20\%$ of randomly chosen features. 
The number of cells to be sampled is determined by a {\em missing rate}, which is defined as the ratio between number of cells with missing values and the number of total cells considered in the dataset.
We generate the incomplete datasets by varying the missing rate from $20\%$ to $80\%$ with an increment of $20\%$. 
To fill cells where values are missing, we consider the {\em active domain} (i.e., values observed in the dataset) of each categorical feature. Given the complete dataset, we randomly divide the dataset into two subsets: 80\% for training and 20\% for testing. Next, the training set is used for model training and we randomly select $1$ to $16$ data points from the testing set as test points.

\begin{figure*}[!ht]
    \centering
    \includegraphics[width=0.96\textwidth]{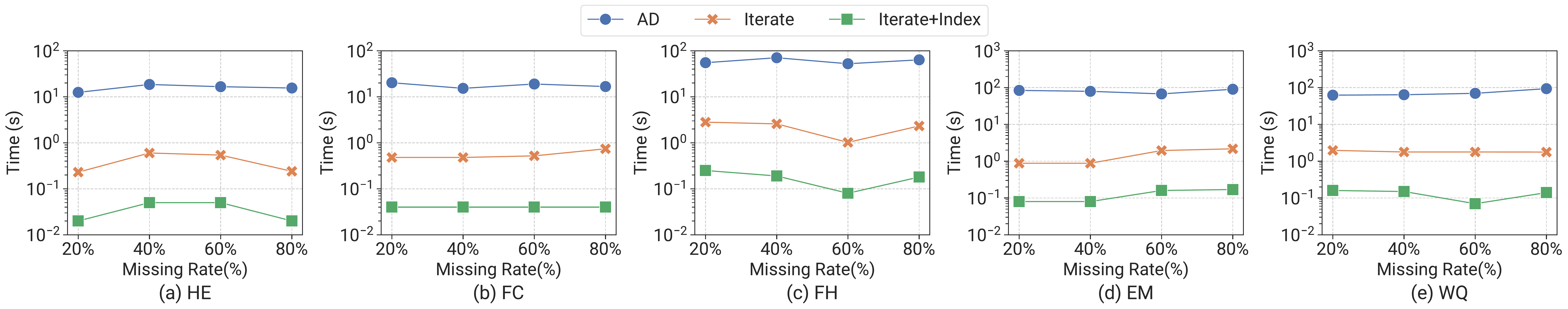}
    % \vspace{-3mm}
    % \begin{center}
    % 	\footnotesize  (a) HE \qquad\qquad\qquad\quad (b) FC \qquad\qquad\qquad\quad (c) FH \qquad\qquad\qquad\quad (d) EM \qquad\qquad\qquad\quad (e) WQ
    % \end{center}
    % \vspace{-2mm}
    \caption{Decision - Running Time vs Missing Rates on Different Datasets}
    \label{fig:decision_time_vs_missing_rate}
    % \vspace{-2mm}
\end{figure*}

\begin{figure*}[!ht]
    \centering
    \includegraphics[width=0.96\textwidth]{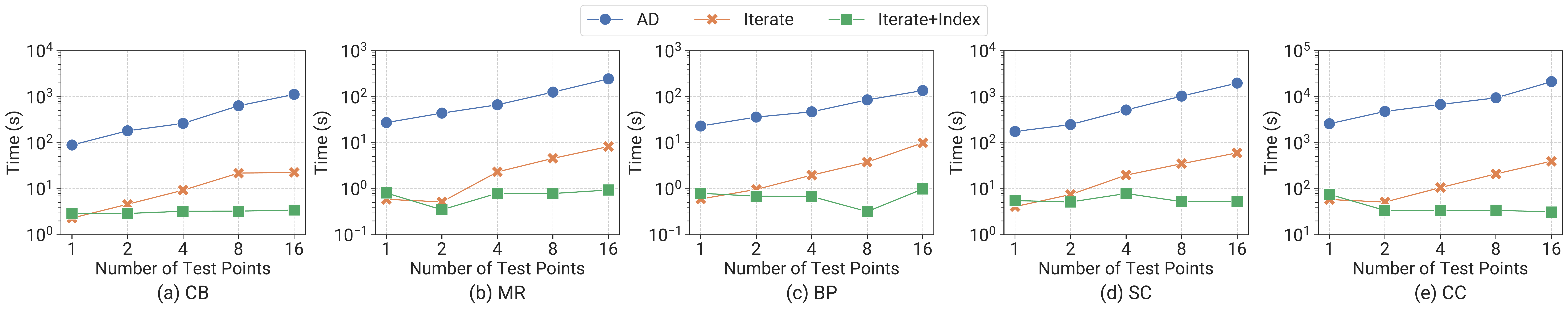}
    % \vspace{-3mm}
    \caption{Decision - Running Time vs Number of Test Points on Different Datasets}
    \label{fig:decision_time_vs_number_of_test_points}
    % \vspace{-3mm}
\end{figure*}

\subsection{Algorithms Evaluated}

% In this section, we introduce all algorithms we evaluate for each problem mentioned in the previous sections. 
% For the \textsf{Decision Problem} and \textsf{Data Poisoning Problem} for NBC, we evaluate all algorithms we introduced in the previous sections. All algorithms are implemented in Python 3.10.9. 
% We left out the evaluation of the algorithm for the \textsf{Counting Problem} for NBC due to its high time complexity.

\subsubsection{Decision Problem.} %We compare the following approaches for the decision problem:
For $\crnbc$, we compare the following algorithms. 

\textbf{Approximate Decision (\textbf{AD}).} This algorithm simply samples $100$ possible worlds uniformly at random, and returns {\em certifiably robust} if the NBC prediction for the test point agrees for every sampled possible world that NBC is trained on, or otherwise returns \emph{certifiably non-robust}. 
Note that this algorithm might return false positive results.

\textbf{Iterative Algorithm (\textbf{Iterate}).} This is the algorithm mentioned in~\cite{ramoni2001robust}. When we are given $k$ test points, the algorithm mentioned~\cite{ramoni2001robust} will run Algorithm~\ref{alg:improved_decision_miss} for $k$ times without using the indexing techniques proposed in Section~\ref{sec:decision}. 
% In particular, when there are $k$ test points, Algorithm~\ref{alg:improved_decision_miss} is executed $k$ times. 

\textbf{Iterative Algorithm with Index (\textbf{Iterate+Index}).} For multiple test points, we run Algorithm~\ref{alg:improved_decision_miss} combined with the indexing techniques described in Section~\ref{sec:decision}.

\subsubsection{Data Poisoning Problem.} %We compare the following approaches for the decision problem:
For $\poisoncrnbc$, we evaluate the following algorithms. 
%\zhiwei{How the robustness in each case is checked? It seems that for SR and GP, it's clear that there should be smarter ways to check the robustness - and we should understand and describe the runtime breakdown for each poisoning algorithm - time spent on poisoning v.s time spent on certifiable robustness checking.}

\textbf{Random Poisoning (\textbf{RP}).} This algorithm randomly selects a cell to be marked as \texttt{NULL} iteratively from the original complete dataset until it produces an incomplete dataset for which $\test$ is certifiably non-robust, which is checked by running \textbf{Iterate}. There is no limit on the number of cells to be selected.

\textbf{Smarter Random Poisoning (\textbf{SR}).} In this algorithm, we first obtain the prediction $l$ of the input test point over the input complete dataset, and randomly fix another label $l' \neq l$. Then, we randomly perform operation A1 or A2, as introduced in Section~\ref{sec:poisoning}, until $\test$ becomes certifiable non-robust, which is checked using Algorithm~\ref{alg:improved_decision_miss}.

\textbf{Greedy Search Poisoning (\textbf{GS}).} This algorithm is described in the Algorithm~\ref{alg:single-poisoning} mentioned in Section~\ref{sec:poisoning}. This algorithm is able to achieve optimal solution when we are given single test point.

When we are given more than one test point, we use the above algorithms in the Line 3 of Algorithm~\ref{alg:multiple-naive} in Appendix~\ref{sec:heuristic-algo-appendix}.

\begin{figure*}[!ht]
    \centering
    \includegraphics[width=0.96\textwidth]{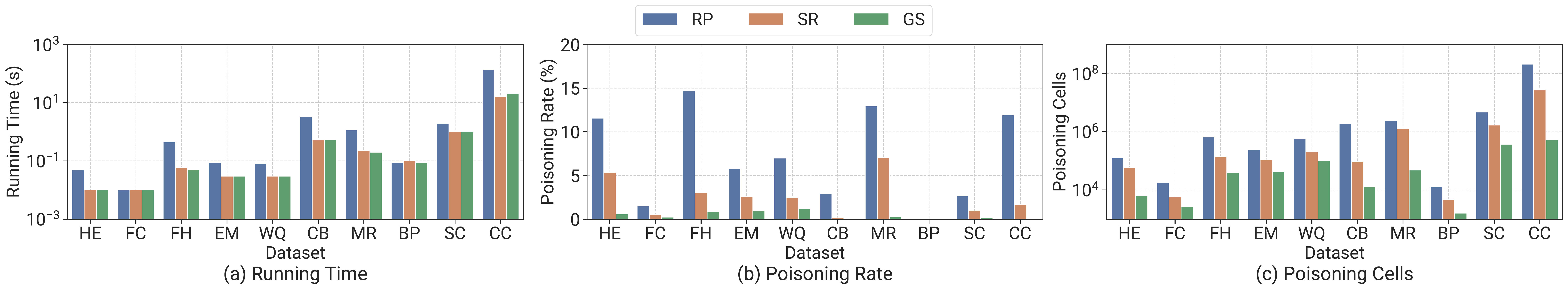}
    % \vspace{-3mm}
    \caption{Single Point Data Poisoning---Running Time and Poisoning Rate on Different Datasets}
    \label{fig:single_point_poisoning}
    % \vspace{-3mm}
\end{figure*}

\subsection{Evaluation Results}
% \noindent \textbf{Evaluation Method.}
% We generate the incomplete datasets by varying the missing rate from $0.002$ to $0.01$ with an increment of 0.002. 
% For each experiment, we vary the number of test points from 1 to 32.  
% The test points are randomly selected from the original complete dataset and the rest of the dataset is used for model training.  

% When there is more than one test point, the decision algorithm needs to certify the robustness of {\em every} test point while for 
% data poisoning, we aim at making at least one test point certifiably non-robust. 
% \xiating{I thought for poisoning, we need all test points to be certifiably non-robust?}
% All algorithms are implemented in Python 3.6. 

% For each dataset, we execute every algorithm $5$ times. 
For both $\crnbc$ and $\poisoncrnbc$, we report the average running time among the $5$ executions for fixed missing rate and the number of points.
For $\poisoncrnbc$, we additionally report the poisoning rate (for poisoning algorithms only), defined as the ratio between the number of cells marked as \texttt{NULL} by the data poisoning algorithm and the total number of cells in the training set. We defer some additional results to Appendix~\ref{sec:more-experiments-appendix}.
% As a general reference, the average training and prediction times of a single test point over each dataset \textsf{FH}, \textsf{BP}, \textsf{MR} and \textsf{SC} are $0.067$s, $0.253$s, $0.331$s and $2.812$s respectively, shown as black dashed lines on Figure 1-6 (i.e., \textbf{Plain}). 
%\xiating{Actually, could we put these lines as a horizontal line in all figures? It would be more readable for my comment in the next section.}

% $$\frac{\text{\# of cells attacked}}{\text{\# of cells in the dataset}}$$
% \begin{table}[!ht]
%     \centering
%     \begin{tabular}{c|c}
%       Dataset & Running Time (s) \\
%       \hline 
%       \textsf{fetal\_health (FH)} & 0.067 \\ 
%       \textsf{bodyPerformance (BP)} & 0.253  \\ 
%       \textsf{Mushroom (MR)} & 0.331 \\
%       \textsf{star\_classification (SC)} & 2.812
%     \end{tabular}
%     \caption{Running time of generating single possible world over various datasets.} 
%     \label{tab:single_possible_world}
%     % \vspace{-6mm}
% \end{table}

\subsubsection{Decision Problems.}
%We compare three decision algorithms: Random Sampling, Baseline Algorithm and Faster Algorithm here. In the first set of experiments, we evaluate the running time of our methods against the baseline methods over various missing rates. Figure~\ref{fig:random_decision_time_vs_missrate} indicates that the decision algorithms are robust to the missing rate, which satisfies the time complexity analysis in Section~\ref{sec:decision}. We also can observe that Faster Algorithm is more efficient than Baseline and Random. For example, when the missing rate is 0.2\% and the dataset is SC, Faster Algorithm is up to 100x (resp. 50x) faster than Baseline (resp. Random). The reason behind it is that Faster Algorithm needs less possible worlds to evaluate $\test$ is certifiably robust. In addition, we also note that Random Sampling Algorithm cannot provide any theoretical guarantees. 
% To answer Q1, 
We first evaluate the running time of different decision algorithms (\textbf{AD}, \textbf{Iterate} and \textbf{Iterate+Index}) for $16$ test points on different datasets over varying missing rates from $20\%$ to $80\%$, with an increment of $20\%$.
% We also include the average running time of training and predicting a single test point for each original complete dataset, shown as \textbf{Plain} as a reference.
From Figure~\ref{fig:decision_time_vs_missing_rate},
we observe that \textbf{Iterate+Index} is almost $20\times$ faster than \textbf{Iterate} and much faster than the straightforward solution \textbf{AD}, which nonetheless does not have correctness guarantees.
We can also see that the efficiency of \textbf{Iterate+Index} is robust against the missing rate in the datasets as the running time of our algorithms is almost uniform.
This shows that it is viable to efficiently apply certifiable robustness in practice to obtain much stronger guarantees in addition to the traditional ML pipeline.
In addition, we run \textbf{AD}, \textbf{Iterate} and \textbf{Iterate+Index} on varying number of test points for each dataset.
Figure~\ref{fig:decision_time_vs_number_of_test_points} shows the running time of each algorithm over datasets obtained by varying from $1$ test point to $16$ test points.
We observe that the running time of \textbf{AD} and \textbf{Iterate} grows almost linearly with the number of test points, whereas the running time of \textbf{Iterate+Index} remains almost the same regardless of the number of test points. 
The initial offline index introduced in the end of Section~\ref{sec:decision} can be slow to build at the beginning (see the sub-Figures when the number of test points is $1$), but its benefits are more apparent when the number of test points increases.
For $16$ test points, we see that \textbf{Iterate+Index} outperforms \textbf{Iterate} by $10\times$.

% When pre-building indexes to pre-compute the conditional probability for all features 
% in the dataset across the active domain before checking the robustness of test points, certifying the robustness of the data batch is much more efficient.
% As shown in ,  the runtime change of \textbf{Iterate+Index} (including the time for building index) is much flatter compared to other algorithms. This suggest 
% that the performance gap between \textbf{Iterate+Index} and other algorithms increases given more test points.

In summary, \textbf{Iterate+Index} is a practical algorithm that can efficiently certify robustness for NBC for multiple test points and large datasets.  %and have a strong linear dependence on the number of test points.
%In summary, our algorithm is orders of magnitude faster than alternatives while providing correctness guarantees, which is the preferred choice for the decision problem over Naive Bayes Classifier.

\subsubsection{Data Poisoning Problem.}
%Next, we conduct two sets of experiments to show the effectiveness and efficiency of our algorithm on data poisoning problem. We compare three algorithms here: Random, SR and Greedy over four real datasets.
%For the first set of experiments, we evaluate the running time of our data poisoning methods against various number of test points. From Figure~\ref{fig:poisoning_time_vs_points}, we can point out that the running time of SR and Greedy is linear to the number of test points. However, Random is not linear to the number of test points. The reason behind it is that we randomly attack each cell in the dataset and we evaluate whether there exists one of the test points is not certifiably robust after attacking in Random Algorithm. If yes, then we stop the whole process. However, as for other methods, we need to evaluate the plan of attacking for each test points. Thus, the running time of SR and Greedy is linear to the number of test points.
%For the second set of experiments, We compare the number of cells we need to attack against the number of test points. From Table~\ref{tab:poisoning_step_vs_points}, we can observe that the number of cells attacked by greedy algorithm is smaller than that of other algorithms, which matches the theoretial analysis in Section~\ref{sec:poisoning}, which shows that the greedy algorithm is able to achieve the optimal solutions.
%In conclusion, our greedy algorithm is able to outperform the baseline in speed and effective. And the experimental results match the theoretial analysis in the previous section.
We now present the evaluation results of different data poisoning algorithms for a single test point. We defer the evaluation of multiple test points to the Appendix~\ref{sec:more-results-poisoning-appendix} due to space constraints.

% Our evaluations to answer Q1 are presented in Figure~\ref{fig:single_poisoning_time}, Figure~\ref{fig:multiple_poisoning_time} and Figure~\ref{fig:multiple_poisoning_cells}. 
In Figure~\ref{fig:single_point_poisoning}, we see that the running time of \textbf{GS} is significantly faster than \textbf{RP} and marginally faster than \textbf{SR} over most datasets. This is because \textbf{SR} uses the optimal strategies A1 and A2 as in \textbf{GS}, but not in the optimal sequence. 

% For Q3, Figure~\ref{fig:single_poisoning_rate} and Figure~\ref{fig:multiple_poisoning_rate} 
Moreover, Figure~\ref{fig:single_point_poisoning} presents the ratio between the number of poisoned cells and the total number of cells, called the poisoning rate. Note that \textbf{RP} and \textbf{SR} always have a higher poisoning rate compared to \textbf{GS} over \textsf{Single Point Data Poisoning Problem}. Since \textbf{GS} is provably optimal when we are given a single test point, it has the smallest poisoning rate across all algorithms for the same dataset. 

Furthermore, Figure~\ref{fig:single_point_poisoning} shows the exact number of poisoned cells. Note that the number of poisoning cells is not small in Figure~\ref{fig:single_point_poisoning}, which indicates that the answer to the \textsf{Decision Problem} can help us to accelerate the data cleaning process.

Finally, note that \textbf{GS} achieves different poisoning rates for different datasets and the poisoning rates are often low. For example, the poisoning rate for \textsf{FC} is much lower than that for \textsf{WQ}, showing that the dataset \textsf{FC} is more prone to a data poisoning attack than \textsf{WQ}. In addition, the experimental results over ten real datasets indicate that NBC is sensitive to the data quality.

% This results can provide us an overview of whether NBC is sensitive to the data quality.
% \xiating{Make sure to decide whether to mention ``sensitivity'' in abstract/intro.}

In conclusion, \textbf{GS} not only provides guarantees for finding the minimum number of poisoned cells when we are given only one test point, but also achieves the best performance in terms of running time and over multiple points case. It also provides insights on the applicability of certifiable robustness.

% \Song{Here, we need to add more description}

% robustness of the dataset against data poisoning attacks.

% \vspace{-2mm}
\section{Related Work}
\label{sec:related}

% \xiating{Should we mention more\dots}

% However, from our results, Naive Bayes Classifiers are sensitive to the data quality. To make test points certifiably robust, we need to clean almost the whole datasets.

% With the development of the Machine Learning techniques, the notion of relational learning over inconsistent databases was also explored~\cite{DBLP:conf/sigmod/PicadoDTL20,DBLP:conf/sigmod/PicadoTP18,DBLP:journals/pvldb/PicadoTP18}. Recently, motivated by the previous works about relational query over incomplete database, CPClean~\cite{klw+20} presents the notion of certain predicion, which connect the database theory and certifiable robustness in ML. In addition, efficient algorithms have been developed for certifying the robustness of $k$-Nearest Neighbors~\cite{fk22}, SVM and linear regression~\cite{DBLP:conf/deem/ZhenCT23}. In this paper, we provide theoretial analysis on certain predictions over Naive Bayes classifiers.

\noindent \textbf{Data Poisoning.} Data poisoning attacks can be divided into triggered and triggerless attacks. As for triggered attacks, they modify the training dataset by adding poison examples, which leads to the models incorrectly categorizing test examples~\cite{cll+17, gdg17, lma+18, ttm19, nt20, ct21, sfc+22}. As for triggerless attacks, they introduce minor adversarial disturbances to the existing training dataset, causing the models to incorrectly classify test examples~\cite{shn+18, zhl+19, gfh+20, amw+21, ylm22}.

\noindent \textbf{Robustness in ML.} Robust learning algorithms have received much attention recently. Robustness of decision tree under adversarial attack has been studied in~\cite{czs+19, vv21} and interests in certifying robust training methods have been observed~\cite{swz+21}. Most recently, efficient algorithms have been developed for certifying the robustness of $k$-Nearest Neighbors~\cite{klw+20}, SVM and linear regression~\cite{skl17, DBLP:conf/deem/ZhenCT23}. 

\noindent \textbf{Querying over Inconsistent Databases.}
Inconsistent database has been studied for several decades. \emph{Consistent query answering} is a principled approach to answer consistent answer from an inconsistent database violating from~\cite{abc99, kw18a, kw18b}. A long line of research has been dedicated to  studying this question from theoretical perspective~\cite{lib11, ks12, kw15, DBLP:conf/pods/KoutrisOW21,10.1145/3651139} and system implementations~\cite{DBLP:conf/sat/DixitK19,DBLP:journals/dke/MarileoB10, manna_ricca_terracina_2015,10.1145/3340531.3411911,DBLP:journals/pacmmod/FanKOW23}. 

\noindent \textbf{Data Cleaning.} 
Cleaning the dirty data is the most natural resolution in the presence of dirty data. There has been a long line of research on data cleaning, such as error detection~\cite{acd+16, mac+19, hmi+19, ir22}, missing value imputation~\cite{tkf+13, wzi+20}, and data deduplication~\cite{cik16, zcd+22, fd21, wzd22}. Multiple data cleaning frameworks have been developed~\cite{wkf+14,kww+16,wfg+17}. For example, SampleClean~\cite{wkf+14} introduces the Sample-and-Clean framework, which answers aggregate query by cleaning a small subset of the data. ActiveClean~\cite{kww+16} is used to clean data for convex models which are trained by using gradient descent. BoostClean~\cite{wfg+17} selects data cleaning algorithms from pre-defined space. Moreover, there also exists some data cleaning frameworks by using knowledge bases and crowdsourcing~\cite{bmn+15, cmi+15, lwz+16}. For instance, KATARA~\cite{cmi+15} resolves ambiguity by using knowledge bases and crowdsourcing marketplaces. Furthermore, there have been efforts to accelerate the data cleaning process~\cite{DBLP:journals/pvldb/RekatsinasCIR17,DBLP:conf/icde/ChuIP13,DBLP:conf/sigmod/ChuMIOP0Y15,DBLP:journals/pvldb/RezigOAEMS21, ysz+20}. Recently, CPClean~\cite{klw+20} provides a data cleaning framework based on certain prediction. 

% \noindent \textbf{Data Poisoning.} 
% Data poisoning attacks on machine learning has been studied for nearly a decade. A large number of work~\cite{npx+14, xbn+15, mz15} shows that data poisoning can easily threaten the integrity of machine learning models. Many of these works focus on a specific class of models such as SVM and linear regression. There are many works proposing robust learning algorithms to defend against data poisoning~\cite{skl17}. 
% We delve into the data poisoning issue in our study to shed light on the practical application of certifiable robustness.
% Our work investigates how to attack the complete dataset to make test points certifiably non-robust.

% \vspace{-2mm}
\section{Conclusion, Limitations, and Future Work}
\label{sec:conclusion}
% In this paper, we extend the notion of certifiable robustness to Naive Bayes classifier. We design and present efficient algorithms certifying robustness for single and multiple data points. 
% We present a polynomial algorithm for the counting problem and show that it has an exponential dependence on the number of attributes and hence efficient counting remains a challenging and open problem. 
% Then we study the data poisoning problem and develop an efficient polynomial algorithm that can find the minimal number of cells on a complete dataset to attack to make a prediction certifiably non-robust. Furthermore, we present an theoretical analysis to show that batch poisoning problem is NP-hard.

% There are many directions that we would like to explore in future work. One direction is to extend the study of certifiable robustness to other popularly used models of more complex structures such as random forest and gradient boosted trees. Additionally, when no efficient exact algorithms can be found, designing proper approximation algorithms with theoretical guarantees is of great interest. Finally, we envision that the insights gained from the study of certifiable robustness of different ML models would help build more efficient data cleaning systems taking into consideration the downstream ML tasks.

In this paper, we study the certifiable robustness of Naive Bayes classifiers over incomplete datasets by solving the Decision Problem and the Data Poisoning Problem algorithmically. 
The experimental results show that all our algorithms exhibit a significant speed-up against the baseline algorithm.

While our results only hold for Naive Bayes Classifiers, an illuminating future direction to extend the study of certifiable robustness to other popularly used models with more complex structures, such as random forest and gradient boosted trees. It is also intriguing to design a general framework to study certifiable robustness for general ML classifiers.
% However, the findings from the poisoning problem indicate that in certain instances, almost the entire dataset needs to be thoroughly cleaned to ensure that the test points fulfill \emph{Certain Predictions}. Therefore, we need to consider more factors before integrating \emph{Certain Predictions} into data cleaning systems.

% In addition, we identify several open problems here: (i) It remains open to study the certifiable robustness of other widely used ML models such as random forests and gradient boosted trees. It would be interesting if we can develop an algorithm for a number of models. 
% (ii) It would be very helpful to design approximation algorithms with theoretical guarantees for $\cntcrnbc$ and $\poisoncrnbc$ for multiple test points. (iii) The insights gained from the study of certifiable robustness of different ML models would lead to more efficient data cleaning systems.

% (ii) It would be very helpful to design approximation algorithms with theoretical guarantees for $\cntcrnbc$ and $\poisoncrnbc$ for multiple test points. 
% \Song{we need to change the description here.}

\section*{Acknowledgements}

We thank the anonymous reviewers for their constructive feedback
and comments. We also thank Austen Z. Fan
for his constructive discussions during the early development of this project.
This work is supported by the National Science Foundation under grant IIS-1910014
and the Anthony C. Klug NCR Fellowship.

\section*{Impact Statement}

This paper presents work whose goal is to advance the field of Machine Learning. There are many potential societal consequences of our work, none which we feel must be specifically highlighted here.

\nocite{langley00}

\bibliography{example_paper}

\begin{thebibliography}{81}
\providecommand{\natexlab}[1]{#1}
\providecommand{\url}[1]{\texttt{#1}}
\expandafter\ifx\csname urlstyle\endcsname\relax
  \providecommand{\doi}[1]{doi: #1}\else
  \providecommand{\doi}{doi: \begingroup \urlstyle{rm}\Url}\fi

\bibitem[web(2022{\natexlab{a}})]{web:kaggle}
{Kaggle}, 2022{\natexlab{a}}.
\newblock URL \url{https://www.kaggle.com/}.

\bibitem[web(2022{\natexlab{b}})]{web:sklearn}
{Sklearn}, 2022{\natexlab{b}}.
\newblock URL \url{https://scikit-learn.org/stable/index.html}.

\bibitem[dat(2023{\natexlab{a}})]{dat:bodyPerformance}
{bodyPerformance Dataset}, 2023{\natexlab{a}}.
\newblock URL \url{https://www.kaggle.com/kukuroo3/body-performance-data}.

\bibitem[dat(2023{\natexlab{b}})]{dat:companybankruptcy}
{Company Bankruptcy Dataset}, 2023{\natexlab{b}}.
\newblock URL
  \url{https://www.kaggle.com/datasets/fedesoriano/company-bankruptcy-prediction}.

\bibitem[dat(2023{\natexlab{c}})]{dat:employee}
{Employee Dataset}, 2023{\natexlab{c}}.
\newblock URL
  \url{https://www.kaggle.com/datasets/tawfikelmetwally/employee-dataset}.

\bibitem[dat(2023{\natexlab{d}})]{dat:fetalhealth}
{fetal\_health Dataset}, 2023{\natexlab{d}}.
\newblock URL
  \url{https://www.kaggle.com/andrewmvd/fetal-health-classification}.

\bibitem[dat(2023{\natexlab{e}})]{dat:heart}
{Heart Dataset}, 2023{\natexlab{e}}.
\newblock URL
  \url{https://www.kaggle.com/datasets/fedesoriano/heart-failure-prediction}.

\bibitem[dat(2023{\natexlab{f}})]{dat:mushroom}
{mushroom Dataset}, 2023{\natexlab{f}}.
\newblock URL \url{https://www.kaggle.com/uciml/mushroom-classification}.

\bibitem[dat(2023{\natexlab{g}})]{dat:winequality}
{Wine Quality Dataset}, 2023{\natexlab{g}}.
\newblock URL \url{https://www.kaggle.com/datasets/rajyellow46/wine-quality}.

\bibitem[Abedjan et~al.(2016)Abedjan, Chu, Deng, Fernandez, Ilyas, Ouzzani,
  Papotti, Stonebraker, and Tang]{acd+16}
Abedjan, Z., Chu, X., Deng, D., Fernandez, R.~C., Ilyas, I.~F., Ouzzani, M.,
  Papotti, P., Stonebraker, M., and Tang, N.
\newblock Detecting data errors: Where are we and what needs to be done?
\newblock \emph{Proceedings of the VLDB Endowment}, 9\penalty0 (12):\penalty0
  993--1004, 2016.

\bibitem[Aghakhani et~al.(2021)Aghakhani, Meng, Wang, Kruegel, and
  Vigna]{amw+21}
Aghakhani, H., Meng, D., Wang, Y.-X., Kruegel, C., and Vigna, G.
\newblock Bullseye polytope: A scalable clean-label poisoning attack with
  improved transferability.
\newblock In \emph{2021 IEEE European symposium on security and privacy
  (EuroS\&P)}, pp.\  159--178. IEEE, 2021.

\bibitem[Arar \& Ayan(2017)Arar and Ayan]{aa17}
Arar, {\"O}.~F. and Ayan, K.
\newblock A feature dependent naive bayes approach and its application to the
  software defect prediction problem.
\newblock \emph{Applied Soft Computing}, 59:\penalty0 197--209, 2017.

\bibitem[Arenas et~al.(1999)Arenas, Bertossi, and Chomicki]{abc99}
Arenas, M., Bertossi, L., and Chomicki, J.
\newblock Consistent query answers in inconsistent databases.
\newblock In \emph{PODS}, 1999.

\bibitem[Arias-Alcaide et~al.(2022)Arias-Alcaide, Soguero-Ruiz, Santos-Alvarez,
  Arche, and Mora-Jim{\'e}nez]{acc+22}
Arias-Alcaide, C., Soguero-Ruiz, C., Santos-Alvarez, P., Arche, J. F.~V., and
  Mora-Jim{\'e}nez, I.
\newblock Local na{\"\i}ve bayes for predicting evolution of covid-19 patients
  on self organizing maps.
\newblock In \emph{2022 IEEE International Conference on Bioinformatics and
  Biomedicine (BIBM)}, pp.\  1443--1450. IEEE, 2022.

\bibitem[Bergman et~al.(2015)Bergman, Milo, Novgorodov, and Tan]{bmn+15}
Bergman, M., Milo, T., Novgorodov, S., and Tan, W.-C.
\newblock Query-oriented data cleaning with oracles.
\newblock In \emph{Proceedings of the 2015 ACM SIGMOD International Conference
  on Management of Data}, pp.\  1199--1214, 2015.

\bibitem[Brooks(1941)]{brooks1941colouring}
Brooks, R.~L.
\newblock On colouring the nodes of a network.
\newblock In \emph{Mathematical Proceedings of the Cambridge Philosophical
  Society}, volume~37, pp.\  194--197. Cambridge University Press, 1941.

\bibitem[Carlini \& Terzis(2021)Carlini and Terzis]{ct21}
Carlini, N. and Terzis, A.
\newblock Poisoning and backdooring contrastive learning.
\newblock \emph{arXiv preprint arXiv:2106.09667}, 2021.

\bibitem[Chen et~al.(2019)Chen, Zhang, Si, Li, Boning, and Hsieh]{czs+19}
Chen, H., Zhang, H., Si, S., Li, Y., Boning, D., and Hsieh, C.-J.
\newblock Robustness verification of tree-based models.
\newblock \emph{arXiv:1906.03849}, 2019.

\bibitem[Chen et~al.(2017)Chen, Liu, Li, Lu, and Song]{cll+17}
Chen, X., Liu, C., Li, B., Lu, K., and Song, D.
\newblock Targeted backdoor attacks on deep learning systems using data
  poisoning.
\newblock \emph{arXiv preprint arXiv:1712.05526}, 2017.

\bibitem[Chu et~al.(2013)Chu, Ilyas, and Papotti]{DBLP:conf/icde/ChuIP13}
Chu, X., Ilyas, I.~F., and Papotti, P.
\newblock Holistic data cleaning: Putting violations into context.
\newblock In \emph{{ICDE}}, pp.\  458--469. {IEEE} Computer Society, 2013.

\bibitem[Chu et~al.(2015{\natexlab{a}})Chu, Morcos, Ilyas, Ouzzani, Papotti,
  Tang, and Ye]{DBLP:conf/sigmod/ChuMIOP0Y15}
Chu, X., Morcos, J., Ilyas, I.~F., Ouzzani, M., Papotti, P., Tang, N., and Ye,
  Y.
\newblock {KATARA:} {A} data cleaning system powered by knowledge bases and
  crowdsourcing.
\newblock In \emph{{SIGMOD} Conference}, pp.\  1247--1261. {ACM},
  2015{\natexlab{a}}.

\bibitem[Chu et~al.(2015{\natexlab{b}})Chu, Morcos, Ilyas, Ouzzani, Papotti,
  Tang, and Ye]{cmi+15}
Chu, X., Morcos, J., Ilyas, I.~F., Ouzzani, M., Papotti, P., Tang, N., and Ye,
  Y.
\newblock Katara: A data cleaning system powered by knowledge bases and
  crowdsourcing.
\newblock In \emph{Proceedings of the 2015 ACM SIGMOD international conference
  on management of data}, pp.\  1247--1261, 2015{\natexlab{b}}.

\bibitem[Chu et~al.(2016{\natexlab{a}})Chu, Ilyas, and Koutris]{cik16}
Chu, X., Ilyas, I.~F., and Koutris, P.
\newblock Distributed data deduplication.
\newblock \emph{Proceedings of the VLDB Endowment}, 9\penalty0 (11):\penalty0
  864--875, 2016{\natexlab{a}}.

\bibitem[Chu et~al.(2016{\natexlab{b}})Chu, Ilyas, Krishnan, and Wang]{cik+16}
Chu, X., Ilyas, I.~F., Krishnan, S., and Wang, J.
\newblock Data cleaning: Overview and emerging challenges.
\newblock In \emph{Proceedings of the 2016 international conference on
  management of data}, pp.\  2201--2206, 2016{\natexlab{b}}.

\bibitem[CloudLab()]{cloudlab}
CloudLab.
\newblock \url{https://www.cloudlab.us/}, 2018.

\bibitem[dee dee(2023)]{dat:fitnessclub}
dee dee.
\newblock {Fitness Club Dataset}, 2023.
\newblock URL \url{https://www.kaggle.com}.

\bibitem[Dixit \& Kolaitis(2019)Dixit and Kolaitis]{DBLP:conf/sat/DixitK19}
Dixit, A.~A. and Kolaitis, P.~G.
\newblock A sat-based system for consistent query answering.
\newblock In \emph{{SAT}}, volume 11628 of \emph{Lecture Notes in Computer
  Science}, pp.\  117--135. Springer, 2019.

\bibitem[Elgiriyewithana(2023)]{dat:creditcard}
Elgiriyewithana, N.
\newblock {Credit Card Dataset}, 2023.
\newblock URL \url{https://www.kaggle.com}.

\bibitem[Fan et~al.(2023)Fan, Koutris, Ouyang, and
  Wijsen]{DBLP:journals/pacmmod/FanKOW23}
Fan, Z., Koutris, P., Ouyang, X., and Wijsen, J.
\newblock Lin{CQA}: Faster consistent query answering with linear time
  guarantees.
\newblock \emph{Proc. {ACM} Manag. Data}, 1\penalty0 (1):\penalty0 38:1--38:25,
  2023.

\bibitem[fedesoriano(2023)]{dat:starclassification}
fedesoriano.
\newblock {star\_classification Dataset}, 2023.
\newblock URL \url{https://www.kaggle.com}.

\bibitem[Feng \& Deng(2021)Feng and Deng]{fd21}
Feng, W. and Deng, D.
\newblock Allign: Aligning all-pair near-duplicate passages in long texts.
\newblock In \emph{Proceedings of the 2021 International Conference on
  Management of Data}, pp.\  541--553, 2021.

\bibitem[Geiping et~al.(2020)Geiping, Fowl, Huang, Czaja, Taylor, Moeller, and
  Goldstein]{gfh+20}
Geiping, J., Fowl, L., Huang, W.~R., Czaja, W., Taylor, G., Moeller, M., and
  Goldstein, T.
\newblock Witches' brew: Industrial scale data poisoning via gradient matching.
\newblock \emph{arXiv preprint arXiv:2009.02276}, 2020.

\bibitem[Gu et~al.(2017)Gu, Dolan-Gavitt, and Garg]{gdg17}
Gu, T., Dolan-Gavitt, B., and Garg, S.
\newblock Badnets: Identifying vulnerabilities in the machine learning model
  supply chain.
\newblock \emph{arXiv preprint arXiv:1708.06733}, 2017.

\bibitem[Heidari et~al.(2019)Heidari, McGrath, Ilyas, and Rekatsinas]{hmi+19}
Heidari, A., McGrath, J., Ilyas, I.~F., and Rekatsinas, T.
\newblock Holodetect: Few-shot learning for error detection.
\newblock In \emph{Proceedings of the 2019 International Conference on
  Management of Data}, pp.\  829--846, 2019.

\bibitem[Ilyas \& Rekatsinas(2022)Ilyas and Rekatsinas]{ir22}
Ilyas, I.~F. and Rekatsinas, T.
\newblock Machine learning and data cleaning: Which serves the other?
\newblock \emph{ACM Journal of Data and Information Quality (JDIQ)},
  14\penalty0 (3):\penalty0 1--11, 2022.

\bibitem[Kalutarage et~al.(2017)Kalutarage, Nguyen, and Shaikh]{kns17}
Kalutarage, H.~K., Nguyen, H.~N., and Shaikh, S.~A.
\newblock Towards a threat assessment framework for apps collusion.
\newblock \emph{Telecommunication systems}, 66\penalty0 (3):\penalty0 417--430,
  2017.

\bibitem[Karla{\v{s}} et~al.(2020)Karla{\v{s}}, Li, Wu, G{\"u}rel, Chu, Wu, and
  Zhang]{klw+20}
Karla{\v{s}}, B., Li, P., Wu, R., G{\"u}rel, N.~M., Chu, X., Wu, W., and Zhang,
  C.
\newblock Nearest neighbor classifiers over incomplete information: From
  certain answers to certain predictions.
\newblock \emph{arXiv:2005.05117}, 2020.

\bibitem[Khalfioui et~al.(2020)Khalfioui, Joertz, Labeeuw, Staquet, and
  Wijsen]{10.1145/3340531.3411911}
Khalfioui, A. A.~E., Joertz, J., Labeeuw, D., Staquet, G., and Wijsen, J.
\newblock Optimization of answer set programs for consistent query answering by
  means of first-order rewriting.
\newblock In \emph{{CIKM}}, pp.\  25--34. {ACM}, 2020.

\bibitem[Koutris \& Suciu(2012)Koutris and Suciu]{ks12}
Koutris, P. and Suciu, D.
\newblock A dichotomy on the complexity of consistent query answering for atoms
  with simple keys.
\newblock \emph{arXiv:1212.6636}, 2012.

\bibitem[Koutris \& Wijsen(2015)Koutris and Wijsen]{kw15}
Koutris, P. and Wijsen, J.
\newblock The data complexity of consistent query answering for self-join-free
  conjunctive queries under primary key constraints.
\newblock In \emph{PODS}, 2015.

\bibitem[Koutris \& Wijsen(2018{\natexlab{a}})Koutris and Wijsen]{kw18a}
Koutris, P. and Wijsen, J.
\newblock Consistent query answering for primary keys and conjunctive queries
  with negated atoms.
\newblock In \emph{Proceedings of the 37th ACM SIGMOD-SIGACT-SIGAI Symposium on
  Principles of Database Systems}, pp.\  209--224, 2018{\natexlab{a}}.

\bibitem[Koutris \& Wijsen(2018{\natexlab{b}})Koutris and Wijsen]{kw18b}
Koutris, P. and Wijsen, J.
\newblock Consistent query answering for primary keys in logspace.
\newblock \emph{arXiv preprint arXiv:1810.03386}, 2018{\natexlab{b}}.

\bibitem[Koutris et~al.(2021)Koutris, Ouyang, and
  Wijsen]{DBLP:conf/pods/KoutrisOW21}
Koutris, P., Ouyang, X., and Wijsen, J.
\newblock Consistent query answering for primary keys on path queries.
\newblock In Libkin, L., Pichler, R., and Guagliardo, P. (eds.), \emph{PODS'21:
  Proceedings of the 40th {ACM} {SIGMOD-SIGACT-SIGAI} Symposium on Principles
  of Database Systems, Virtual Event, China, June 20-25, 2021}, pp.\  215--232.
  {ACM}, 2021.
\newblock \doi{10.1145/3452021.3458334}.
\newblock URL \url{https://doi.org/10.1145/3452021.3458334}.

\bibitem[Koutris et~al.(2024)Koutris, Ouyang, and Wijsen]{10.1145/3651139}
Koutris, P., Ouyang, X., and Wijsen, J.
\newblock Consistent query answering for primary keys on rooted tree queries.
\newblock \emph{Proc. ACM Manag. Data}, 2\penalty0 (2), may 2024.
\newblock \doi{10.1145/3651139}.
\newblock URL \url{https://doi.org/10.1145/3651139}.

\bibitem[Krishnan et~al.(2016)Krishnan, Wang, Wu, Franklin, and
  Goldberg]{kww+16}
Krishnan, S., Wang, J., Wu, E., Franklin, M.~J., and Goldberg, K.
\newblock Activeclean: Interactive data cleaning for statistical modeling.
\newblock \emph{Proceedings of the VLDB Endowment}, 2016.

\bibitem[Krishnan et~al.(2017)Krishnan, Franklin, Goldberg, and Wu]{wfg+17}
Krishnan, S., Franklin, M.~J., Goldberg, K., and Wu, E.
\newblock Boostclean: Automated error detection and repair for machine
  learning.
\newblock \emph{arXiv preprint arXiv:1711.01299}, 2017.

\bibitem[Langley(2000)]{langley00}
Langley, P.
\newblock Crafting papers on machine learning.
\newblock In Langley, P. (ed.), \emph{Proceedings of the 17th International
  Conference on Machine Learning (ICML 2000)}, pp.\  1207--1216, Stanford, CA,
  2000. Morgan Kaufmann.

\bibitem[Li et~al.(2016)Li, Wang, Zheng, and Franklin]{lwz+16}
Li, G., Wang, J., Zheng, Y., and Franklin, M.~J.
\newblock Crowdsourced data management: A survey.
\newblock \emph{IEEE Transactions on Knowledge and Data Engineering},
  28\penalty0 (9):\penalty0 2296--2319, 2016.

\bibitem[Libkin(2011)]{lib11}
Libkin, L.
\newblock Incomplete information and certain answers in general data models.
\newblock In \emph{PODS}, 2011.

\bibitem[Liu et~al.(2018)Liu, Ma, Aafer, Lee, Zhai, Wang, and Zhang]{lma+18}
Liu, Y., Ma, S., Aafer, Y., Lee, W.-C., Zhai, J., Wang, W., and Zhang, X.
\newblock Trojaning attack on neural networks.
\newblock In \emph{25th Annual Network And Distributed System Security
  Symposium (NDSS 2018)}. Internet Soc, 2018.

\bibitem[Mahdavi et~al.(2019)Mahdavi, Abedjan, Castro~Fernandez, Madden,
  Ouzzani, Stonebraker, and Tang]{mac+19}
Mahdavi, M., Abedjan, Z., Castro~Fernandez, R., Madden, S., Ouzzani, M.,
  Stonebraker, M., and Tang, N.
\newblock Raha: A configuration-free error detection system.
\newblock In \emph{Proceedings of the 2019 International Conference on
  Management of Data}, pp.\  865--882, 2019.

\bibitem[Manna et~al.(2015)Manna, Ricca, and
  Terracina]{manna_ricca_terracina_2015}
Manna, M., Ricca, F., and Terracina, G.
\newblock Taming primary key violations to query large inconsistent data via
  {ASP}.
\newblock \emph{Theory Pract. Log. Program.}, 15\penalty0 (4-5):\penalty0
  696--710, 2015.

\bibitem[Marileo \& Bertossi(2010)Marileo and
  Bertossi]{DBLP:journals/dke/MarileoB10}
Marileo, M.~C. and Bertossi, L.~E.
\newblock The consistency extractor system: Answer set programs for consistent
  query answering in databases.
\newblock \emph{Data Knowl. Eng.}, 69\penalty0 (6):\penalty0 545--572, 2010.

\bibitem[Nguyen \& Tran(2020)Nguyen and Tran]{nt20}
Nguyen, T.~A. and Tran, A.
\newblock Input-aware dynamic backdoor attack.
\newblock \emph{Advances in Neural Information Processing Systems},
  33:\penalty0 3454--3464, 2020.

\bibitem[Peng et~al.(2021)Peng, Wu, Lockhart, Bian, Yan, Xu, Chi, Rzeszotarski,
  and Wang]{pwl+21}
Peng, J., Wu, W., Lockhart, B., Bian, S., Yan, J.~N., Xu, L., Chi, Z.,
  Rzeszotarski, J.~M., and Wang, J.
\newblock Dataprep. eda: task-centric exploratory data analysis for statistical
  modeling in python.
\newblock In \emph{SIGMOD}, pp.\  2271--2280, 2021.

\bibitem[Picado et~al.(2020)Picado, Davis, Termehchy, and
  Lee]{DBLP:conf/sigmod/PicadoDTL20}
Picado, J., Davis, J., Termehchy, A., and Lee, G.~Y.
\newblock Learning over dirty data without cleaning.
\newblock In Maier, D., Pottinger, R., Doan, A., Tan, W., Alawini, A., and Ngo,
  H.~Q. (eds.), \emph{Proceedings of the 2020 International Conference on
  Management of Data, {SIGMOD} Conference 2020, online conference [Portland,
  OR, USA], June 14-19, 2020}, pp.\  1301--1316. {ACM}, 2020.
\newblock \doi{10.1145/3318464.3389708}.
\newblock URL \url{https://doi.org/10.1145/3318464.3389708}.

\bibitem[Ramoni \& Sebastiani(2001)Ramoni and Sebastiani]{ramoni2001robust}
Ramoni, M. and Sebastiani, P.
\newblock Robust bayes classifiers.
\newblock \emph{Artificial Intelligence}, 125\penalty0 (1-2):\penalty0
  209--226, 2001.

\bibitem[Razaque et~al.(2017)Razaque, Soomro, Shaikh, Soomro, Samo, Kumar, and
  Dharejo]{rss+17}
Razaque, F., Soomro, N., Shaikh, S.~A., Soomro, S., Samo, J.~A., Kumar, N., and
  Dharejo, H.
\newblock Using na{\"\i}ve bayes algorithm to students' bachelor academic
  performances analysis.
\newblock In \emph{2017 4th IEEE International Conference on Engineering
  Technologies and Applied Sciences (ICETAS)}, pp.\  1--5. IEEE, 2017.

\bibitem[Rekatsinas et~al.(2017)Rekatsinas, Chu, Ilyas, and
  R{\'{e}}]{DBLP:journals/pvldb/RekatsinasCIR17}
Rekatsinas, T., Chu, X., Ilyas, I.~F., and R{\'{e}}, C.
\newblock Holoclean: Holistic data repairs with probabilistic inference.
\newblock \emph{Proc. {VLDB} Endow.}, 10\penalty0 (11):\penalty0 1190--1201,
  2017.

\bibitem[Rezig et~al.(2021)Rezig, Ouzzani, Aref, Elmagarmid, Mahmood, and
  Stonebraker]{DBLP:journals/pvldb/RezigOAEMS21}
Rezig, E.~K., Ouzzani, M., Aref, W.~G., Elmagarmid, A.~K., Mahmood, A.~R., and
  Stonebraker, M.
\newblock Horizon: Scalable dependency-driven data cleaning.
\newblock \emph{Proc. {VLDB} Endow.}, 14\penalty0 (11):\penalty0 2546--2554,
  2021.

\bibitem[Sahami et~al.(1998)Sahami, Dumais, Heckerman, and Horvitz]{sdh+98}
Sahami, M., Dumais, S., Heckerman, D., and Horvitz, E.
\newblock A bayesian approach to filtering junk e-mail.
\newblock In \emph{Learning for Text Categorization: Papers from the 1998
  workshop}, volume~62, pp.\  98--105. Citeseer, 1998.

\bibitem[Shafahi et~al.(2018)Shafahi, Huang, Najibi, Suciu, Studer, Dumitras,
  and Goldstein]{shn+18}
Shafahi, A., Huang, W.~R., Najibi, M., Suciu, O., Studer, C., Dumitras, T., and
  Goldstein, T.
\newblock Poison frogs! targeted clean-label poisoning attacks on neural
  networks.
\newblock \emph{Advances in neural information processing systems}, 31, 2018.

\bibitem[Shen et~al.(2019)Shen, Li, Zheng, Tang, and Yang]{slz+19}
Shen, Y., Li, Y., Zheng, H.-T., Tang, B., and Yang, M.
\newblock Enhancing ontology-driven diagnostic reasoning with a
  symptom-dependency-aware na{\"\i}ve bayes classifier.
\newblock \emph{BMC bioinformatics}, 20:\penalty0 1--14, 2019.

\bibitem[Shi et~al.(2021)Shi, Wang, Zhang, Yi, and Hsieh]{swz+21}
Shi, Z., Wang, Y., Zhang, H., Yi, J., and Hsieh, C.-J.
\newblock Fast certified robust training with short warmup.
\newblock \emph{NeurIPS}, 2021.

\bibitem[Souri et~al.(2022)Souri, Fowl, Chellappa, Goldblum, and
  Goldstein]{sfc+22}
Souri, H., Fowl, L., Chellappa, R., Goldblum, M., and Goldstein, T.
\newblock Sleeper agent: Scalable hidden trigger backdoors for neural networks
  trained from scratch.
\newblock \emph{Advances in Neural Information Processing Systems},
  35:\penalty0 19165--19178, 2022.

\bibitem[Steinhardt et~al.(2017)Steinhardt, Koh, and Liang]{skl17}
Steinhardt, J., Koh, P. W.~W., and Liang, P.~S.
\newblock Certified defenses for data poisoning attacks.
\newblock \emph{NIPS}, 2017.

\bibitem[Trushkowsky et~al.(2013)Trushkowsky, Kraska, Franklin, and
  Sarkar]{tkf+13}
Trushkowsky, B., Kraska, T., Franklin, M.~J., and Sarkar, P.
\newblock Crowdsourced enumeration queries.
\newblock In \emph{2013 IEEE 29th International Conference on Data Engineering
  (ICDE)}, pp.\  673--684. IEEE, 2013.

\bibitem[Turner et~al.(2019)Turner, Tsipras, and Madry]{ttm19}
Turner, A., Tsipras, D., and Madry, A.
\newblock Label-consistent backdoor attacks.
\newblock \emph{arXiv preprint arXiv:1912.02771}, 2019.

\bibitem[Veni \& Srinivasan(2017)Veni and Srinivasan]{vs17}
Veni, S. and Srinivasan, A.
\newblock Defect classification using na{\"\i}ve bayes classification.
\newblock \emph{Interbational Journal of Applied Engineering Research},
  12\penalty0 (22):\penalty0 12693--12700, 2017.

\bibitem[Vos \& Verwer(2021)Vos and Verwer]{vv21}
Vos, D. and Verwer, S.
\newblock Efficient training of robust decision trees against adversarial
  examples.
\newblock In \emph{ICML}, 2021.

\bibitem[Wan(2023)]{wan23}
Wan, C.
\newblock Predicting the effect of genes on longevity with novel hierarchical
  dependency-constrained tree augmented na{\"\i}ve bayes classifiers.
\newblock In \emph{2023 IEEE International Conference on Bioinformatics and
  Biomedicine (BIBM)}, pp.\  926--929. IEEE, 2023.

\bibitem[Wang et~al.(2014)Wang, Krishnan, Franklin, Goldberg, Kraska, and
  Milo]{wkf+14}
Wang, J., Krishnan, S., Franklin, M.~J., Goldberg, K., Kraska, T., and Milo, T.
\newblock A sample-and-clean framework for fast and accurate query processing
  on dirty data.
\newblock In \emph{Proceedings of the 2014 ACM SIGMOD international conference
  on Management of data}, pp.\  469--480, 2014.

\bibitem[Wang et~al.(2022)Wang, Zuo, and Deng]{wzd22}
Wang, Z., Zuo, C., and Deng, D.
\newblock Txtalign: efficient near-duplicate text alignment search via bottom-k
  sketches for plagiarism detection.
\newblock In \emph{Proceedings of the 2022 International Conference on
  Management of Data}, pp.\  1146--1159, 2022.

\bibitem[Wu et~al.(2020)Wu, Zhang, Ilyas, and Rekatsinas]{wzi+20}
Wu, R., Zhang, A., Ilyas, I., and Rekatsinas, T.
\newblock Attention-based learning for missing data imputation in holoclean.
\newblock \emph{Proceedings of Machine Learning and Systems}, 2:\penalty0
  307--325, 2020.

\bibitem[Xiong et~al.(2021)Xiong, Bao, Zhao, and Wang]{xbz+21}
Xiong, T., Bao, Y., Zhao, P., and Wang, Y.
\newblock Covariance estimation and its application in large-scale online
  controlled experiments.
\newblock \emph{arXiv preprint arXiv:2108.02668}, 2021.

\bibitem[Yan et~al.(2020)Yan, Schulte, Zhang, Wang, and Cheng]{ysz+20}
Yan, J.~N., Schulte, O., Zhang, M., Wang, J., and Cheng, R.
\newblock Scoded: Statistical constraint oriented data error detection.
\newblock In \emph{SIGMOD}, pp.\  845--860, 2020.

\bibitem[Yang et~al.(2022)Yang, Liu, and Mirzasoleiman]{ylm22}
Yang, Y., Liu, T.~Y., and Mirzasoleiman, B.
\newblock Not all poisons are created equal: Robust training against data
  poisoning.
\newblock In \emph{International Conference on Machine Learning}, pp.\
  25154--25165. PMLR, 2022.

\bibitem[Yu et~al.(2019)Yu, Xuan, Feng, Zou, and Wang]{yxf+19}
Yu, J., Xuan, Z., Feng, X., Zou, Q., and Wang, L.
\newblock A novel collaborative filtering model for lncrna-disease association
  prediction based on the na{\"\i}ve bayesian classifier.
\newblock \emph{BMC bioinformatics}, 20:\penalty0 1--13, 2019.

\bibitem[Zhen et~al.(2023)Zhen, Chabada, and
  Termehchy]{DBLP:conf/deem/ZhenCT23}
Zhen, C., Chabada, A.~S., and Termehchy, A.
\newblock When can we ignore missing data in model training?
\newblock In \emph{Proceedings of the Seventh Workshop on Data Management for
  End-to-End Machine Learning, {DEEM} 2023, Seattle, WA, USA, 18 June 2023},
  pp.\  4:1--4:4. {ACM}, 2023.
\newblock \doi{10.1145/3595360.3595854}.
\newblock URL \url{https://doi.org/10.1145/3595360.3595854}.

\bibitem[Zhou et~al.(2022)Zhou, Chen, Das, Min, Yu, Zhao, and Zou]{zcd+22}
Zhou, L., Chen, J., Das, A., Min, H., Yu, L., Zhao, M., and Zou, J.
\newblock Serving deep learning models with deduplication from relational
  databases.
\newblock \emph{arXiv preprint arXiv:2201.10442}, 2022.

\bibitem[Zhu et~al.(2019)Zhu, Huang, Li, Taylor, Studer, and Goldstein]{zhl+19}
Zhu, C., Huang, W.~R., Li, H., Taylor, G., Studer, C., and Goldstein, T.
\newblock Transferable clean-label poisoning attacks on deep neural nets.
\newblock In \emph{International Conference on Machine Learning}, pp.\
  7614--7623. PMLR, 2019.

\end{thebibliography}
\bibliographystyle{icml2024}

%%%%%%%%%%%%%%%%%%%%%%%%%%%%%%%%%%%%%%%%%%%%%%%%%%%%%%%%%%%%%%%%%%%%%%%%%%%%%%%
%%%%%%%%%%%%%%%%%%%%%%%%%%%%%%%%%%%%%%%%%%%%%%%%%%%%%%%%%%%%%%%%%%%%%%%%%%%%%%%
% APPENDIX
%%%%%%%%%%%%%%%%%%%%%%%%%%%%%%%%%%%%%%%%%%%%%%%%%%%%%%%%%%%%%%%%%%%%%%%%%%%%%%%
%%%%%%%%%%%%%%%%%%%%%%%%%%%%%%%%%%%%%%%%%%%%%%%%%%%%%%%%%%%%%%%%%%%%%%%%%%%%%%%
\newpage
\appendix
\onecolumn

\section{Missing Details in Section~\ref{sec:decision}}

\subsection{Proof of Lemma~\ref{lemma:decision}}
\label{sec:decision-proof-appendix}
% \begin{lemma}
%     A test point $\test$ is certifiably robust for NBC over $\idata$ if and only if there is a label $l$ such that for any label $l' \neq l$, $\minsupport{\idata}{l}{\test} > \maxsupport{\idata}{l'}{\test}.$
% \end{lemma}
\begin{proof}
It is easy to see that $\test$ is certifiably robust if there exists some label $l$ such that for any label $l' \neq l$, 
$\minsupport{\idata}{l}{\test} > \maxsupport{\idata}{l'}{\test}.$
In this case, the label $l$ will always be predicted for $\test$. 
Indeed, for any possible world  $\mathcal{D}$ of $\idata$, we have 
$$\support{\mathcal{D}}{l}{\test} \geq \minsupport{\idata}{l}{\test} > \maxsupport{\idata}{l'}{\test} \geq \support{\mathcal{D}}{l'}{\test},$$
and thus $f_{\mathcal{D}}(\test) = l$. To help the readers understand the intuition, we give an example in Appendix~\ref{sec:decision-example-appendix}.

The other direction holds somewhat nontrivially: if $\test$ is certifiably robust, then we claim that there must be a label $l$ such that for any label $l' \neq l$, $\minsupport{\idata}{l}{\test} > \maxsupport{\idata}{l'}{\test}.$
Assume that $\test$ is certifiably robust. Let $\mathcal{D}$ be a possible world of $\idata$, and let $l = f_{\mathcal{D}}(\test)$.
Suppose for contradiction that there is a label $l' \neq l$ such that 
$$\minsupport{\idata}{l}{\test} \leq \maxsupport{\idata}{l'}{\test}.$$
Let $\mathcal{D}_1$ and $\mathcal{D}_2$ be possible worlds of $\idata$ such that 
$$\support{\mathcal{D}_1}{l}{\test} = \minsupport{\idata}{l}{\test}$$
and
$$\support{\mathcal{D}_2}{l'}{\test} = \maxsupport{\idata}{l'}{\test}.$$
Consider an arbitrary possible world $\mathcal{D}^*$ of $\idata$ that contains all data points in $\mathcal{D}_1$ with label $l$, all data points in $\mathcal{D}_2$ with label $l'$.
It is easy to verify that $\support{\mathcal{D}^*}{l}{\test} = \support{\mathcal{D}_1}{l}{\test}$ and $\support{\mathcal{D}^*}{l'}{\test} = \support{\mathcal{D}_2}{l'}{\test}$.
We then have
\begin{align*}
\support{\mathcal{D}^*}{l}{\test} &= \support{\mathcal{D}_1}{l}{\test} = \minsupport{\idata}{l}{\test} \\
&\leq \maxsupport{\idata}{l'}{\test}= \support{\mathcal{D}_2}{l'}{\test} = \support{\mathcal{D}^*}{l'}{\test}
\end{align*} 
and therefore 
$f_{\mathcal{D}^*}(\test) \neq l$, a contradiction to that $\test$ is certifiably robust for NBC over $\idata$.
\end{proof}

\subsection{Additional Examples}

\begin{figure*}
\centering
\subfloat[An incomplete dataset $\idata$]{
\adjustbox{valign=b}{
\begin{tabular}{cc}
\begin{tabular}{c | c | c }
$X$ & $Y$ &  \\
\hline
$a$ & $b$ & \multirow{4}{1pt}{$l^*$}\\
$c$ & $b$ \\
$a$ & $d$ \\
$a$ & $\square$
\end{tabular}
&
\begin{tabular}{c | c | c}
$X$ & $Y$ &  \\
\hline
$a$ & $b$ & \multirow{5}{1pt}{$l$}\\
$b$ & $c$ & \\
$\square$ & $\bot_4$ & \\
$\square$ & $\bot_5$ & \\
$\bot_6$ & $\bot_7$ &
\end{tabular}
\end{tabular}
}
}
\quad
\subfloat[A possible world $\mathcal{D}$ of $\idata$.]{
\adjustbox{valign=b}{
\begin{tabular}{cc}
\begin{tabular}{c | c | c }
$X$ & $Y$ &  \\
\hline
$a$ & $b$ & \multirow{4}{1pt}{$l^*$}\\
$c$ & $b$ \\
$a$ & $d$ \\
$a$ & \framebox{$b$}
\end{tabular}
&
\begin{tabular}{c | c | c}
$X$ & $Y$ &  \\
\hline
$a$ & $b$ & \multirow{5}{1pt}{$l$}\\
$b$ & $c$ & \\
\framebox{$a$} & $\bot_1$ & \\
\framebox{$b$} & $\bot_2$ & \\
$\bot_3$ & $\bot_4$ &
\end{tabular}
\end{tabular}
}
}
\quad
\subfloat[An extreme possible world $\mathcal{D}^*$ of $\idata$.]{
\adjustbox{valign=b}{
\begin{tabular}{cc}
\begin{tabular}{c | c | c }
$X$ & $Y$ &  \\
\hline
$a$ & $b$ & \multirow{4}{1pt}{$l^*$}\\
$c$ & $b$ \\
$a$ & $d$ \\
$a$ & \framebox{$d$}
\end{tabular}
&
\begin{tabular}{c | c | c}
$X$ & $Y$ &  \\
\hline
$a$ & $b$ & \multirow{5}{1pt}{$l$}\\
$b$ & $c$ & \\
\framebox{$a$} & $\bot_1$ & \\
\framebox{$a$} & $\bot_2$ & \\
$\bot_3$ & $\bot_4$ &
\end{tabular}
\end{tabular}
}
}
\caption{
The incomplete data $\idata$ has (among others) two possible worlds $\mathcal{D}$ and $\mathcal{D}^*$ shown in (b) and (c) respectively.
}
\label{fig:decision-intuition}
\end{figure*}

\label{sec:decision-example-appendix}
\begin{example}
\label{ex:decision}
Consider the incomplete dataset $\idata$ in Figure~\ref{fig:decision-intuition}(a) and a test point $\test = (a, b)$.
For a possible world $\mathcal{D}$ of $\idata$, we have that 
$$\support{\mathcal{D}}{l^*}{\test} = \frac{4}{9} \cdot \frac{3}{4} \cdot \frac{3}{4} \quad\text{  and  }\quad \support{\mathcal{D}}{l}{\test} = \frac{5}{9} \cdot \frac{2}{5} \cdot \frac{1}{5}.$$

For another possible world $\mathcal{D}^*$ of $\idata$, we have
$$\support{\mathcal{D}^*}{l^*}{\test} = \frac{4}{9} \cdot \frac{3}{4} \cdot \frac{2}{4} \quad\text{  and  }\quad \support{\mathcal{D}^*}{l}{\test} = \frac{5}{9} \cdot \frac{3}{5} \cdot \frac{1}{5}.$$

Among all possible worlds of $\idata$, $\mathcal{D}^*$ has the lowest possible support value $\support{\mathcal{D}}{l^*}{\test}$ and the highest support value $\support{\mathcal{D}}{l}{\test}$ for $l$. But still, $$\support{\mathcal{D}^*}{l^*}{\test} > \support{\mathcal{D}^*}{l}{\test},$$ so $l^*$ will always be predicted regardless of which possible world of $\idata$ NBC is trained on.
\end{example}

\section{Missing Details in Section~\ref{sec:poisoning}}
\label{sec:poisoning-appendix}

\subsection{Proof of Lemma~\ref{lemma:poisoning-complete}}
\label{sec:poisoning-complete-appendix}
% \begin{lemma}
%     Let $\mathcal{D}$ be a complete dataset and $\test$ a test point. Let $l^* = f_{\mathcal{D}}(\test)$. Let $k$ be an integer. Then the following statements are equivalent:
%     \begin{enumerate}
%         \item There exists a solution $\mathcal{D}^{\dagger}$ for $\poisoncrnbc$ for a complete dataset $\mathcal{D}$ and a test point $\test$ with $k$ altered cells.
%         \item There exists a label $l\neq l^*$ and a solution $\mathcal{D}'$ for \textsf{AlterPrediction} for a complete dataset $\mathcal{D}$, a test point $\test$ and label $l^*$ with $k$ altered cells. 
%     \end{enumerate}
% \end{lemma}
\begin{proof}

We first establish the following claim.
\begin{lemma}
\label{claim:int}
Let $D$ be a complete dataset and $t$ a test point. Let $l^* = f_{D}(t)$. Let k be an integer. Then the following statements are equivalent:
\begin{enumerate}
\item there exists an incomplete dataset $D^+$ with $k$ missing cells for which $t$ is certifiably non-robust; and
\item there exists a label $l \neq l^*$ and a complete dataset $D’$ obtained by altering $k$ cells in D such that $f_{D’}(t) = l$.
\end{enumerate}
\end{lemma}
\begin{proof}
If (1) holds, let $\mathcal{D}^{\dagger}$ be a solution $\mathcal{D}^{\dagger}$ for $\poisoncrnbc$ on input $(\mathcal{D}, \test)$ with $k$ altered cells.
Since $\mathcal{D}$ is a possible world of $\mathcal{D}^{\dagger}$, then there exists a possible world $\mathcal{D}'$ of $\mathcal{D}^{\dagger}$ such that $l = f_{\mathcal{D}'}(\test) \neq l^*,$ and therefore
$$\support{\mathcal{D}'}{l}{\test} > \support{\mathcal{D}'}{l^*}{\test}.$$

If (2) holds, there exists some label $l \neq l^*$, and a complete dataset $\mathcal{D}'$ obtained by altering $k$ cells in $\mathcal{D}$ so that 
$$\support{\mathcal{D}'}{l}{\test} > \support{\mathcal{D}'}{l^*}{\test}.$$
Hence 
$f_{\mathcal{D}'}(\test) \neq l^*$. Consider the incomplete dataset $\mathcal{D}^{\dagger}$ obtained by setting all the altered cells in $\mathcal{D}'$ to \texttt{NULL}. Then $\test$ is certifiably non-robust, since for the two possible worlds $\mathcal{D}$ and $\mathcal{D}'$, $f_{\mathcal{D}'}(\test) \neq l^*$ and  $f_{\mathcal{D}}(\test) = l^*$.
\end{proof}

To show that Lemma 4.2 holds, note that the two items in Claim~\ref{claim:int} share the same integer parameter $k$, and thus the number of missing cells $k$ for $\poisoncrnbc$ is minimized simultaneously as we minimize the number of altered cells for \textsf{AlterPrediction}.
\end{proof}

\subsection{Proof of Observations O1, O2, and O3}
\label{sec:observation-args}

All three observations can also be formally explained from Equation~(\ref{eq:product}). Suppose that $ \Pr(x_j \mid l)_{\mathcal{D}_i}$ is the smallest among all $1 \leq j \leq d$. 
If we apply A2 to decrease $ \Pr(x_j \mid l^*)_{\mathcal{D}_i}$ in $\mathcal{D}_i$ and obtain $\mathcal{D}_{i+1}$, then $ \Pr(x_j \mid l^*)_{\mathcal{D}_{i+1}}$ is still the smallest among all $1 \leq j \leq d$, and thus the reduction in $ \Pr(\test \mid l^*)_{\mathcal{D}_i}$ remains the same since we always decrease the same value at the same attribute. 
However, if we apply A1 to increase $\Pr(x_j \mid l)_{\mathcal{D}_i}$ in $\mathcal{D}_i$ and obtain $\mathcal{D}_{i+1}$, then $\Pr(x_j \mid l)_{\mathcal{D}_{i+1}}$ may not be the smallest among all $1 \leq j \leq d$. If it still is, then the next increase in $\Pr(\test \mid l)_{\mathcal{D}_{i+1}}$ would be the same, or otherwise the increase in $\Pr(\test \mid l)_{\mathcal{D}_i}$ would grow.
Since $l \neq l^*$, the increase in $\Pr(\test \mid l)_{\mathcal{D}_i}$ by applying A1 does not change the value of $\Pr(\test \mid l^*)_{\mathcal{D}_i}$, and thus the increase only depends on the number of times that A1 has already been applied to obtain $\mathcal{D}_i$. 
A similar argument also holds for A2, and thus O3 holds.

\subsection{Proof of Lemma~\ref{lemma:key}}
\label{sec:proof-key-lemma}

\begin{proof}
By definition, we have $\delta_{1}^+ \leq \delta_2^+ \leq \dots $.

We are now ready to prove our key result: $\Delta_k$ can \emph{always} be minimized by applying only A1 or A2 to the orignal dataset $\mathcal{D}$, or equivalently,

To show Equation~(\ref{eq:goal}), assume that $\mathcal{D}_k$ is obtained by applying A1 $i^+$ times and applying A2 $i^-$ times, where $k = i^+ + i^-$. 
Then, 
$$\Delta_{k} = \Delta_0 - i^- \cdot \delta^- - \sum_{j \leq i^+} \delta_j^+.$$
Let $k^*\in\{0,1,2,\dots\}$ be the largest such that 
$$k^* \cdot \delta^- \geq \sum_{j \leq k^*} \delta_j^+.$$
Then it must be that $\delta_{k^*+1}^+ > \delta^-$.

If $i^+ \leq k^*$, then we have
\begin{align*}
\Delta_{k} &~= \Delta_0 - i^- \cdot \delta^- - \sum_{j \leq i^+} \delta_j^+ \\
&~\geq \Delta_0 - (i^- + i^+) \cdot \delta^- = \Delta_0 - k \cdot \delta^-
\end{align*}
and if $i^+ > k^*$, we have for any $j \geq i^+$, $\delta_j^+ \geq \delta_{k^*+1}^+ > \delta^-$, and thus
\begin{align*}
\Delta_{k} &~= \Delta_0 - i^- \cdot \delta^- - \sum_{j =1}^{i^+} \delta_j^+ \\
&~\geq \Delta_0 - \sum_{j = i^+ + 1}^{i^+ + i^-} \delta_j^+ - \sum_{j=1}^{i^+} \delta_j^+ = \Delta_0 - \sum_{j\leq k} \delta^+_j
\end{align*}
as desired.
\end{proof}

\subsection{A Heuristic Algorithm for Poisoning Multiple Test Points}

\label{sec:heuristic-algo-appendix}

\begin{algorithm}[!ht]
\caption{$\poisoncrnbc$-Multiple}
\label{alg:multiple-naive}
\KwIn{A complete dataset $\mathcal{D}$, $k$ test points $\test_1$, $\test_2$, \dots, $\test_k$}
\KwOut{An incomplete dataset $\mathcal{D}^{\dagger}$ obtained by setting some cells in $\mathcal{D}$ to \texttt{NULL} such that every $\test_i$ is not certifiably-robust for NBC }
$S \leftarrow \emptyset$ \\
\ForEach{$i = 1, 2, \dots k$}{ 
    $\mathcal{D}^{\dagger}_i \leftarrow \poisoncrnbc\text{-Single}(\mathcal{D}, \test_i)$ \\
    add the missing cells in $\mathcal{D}^{\dagger}_i$ to $S$ 
}
$\mathcal{D}^{\dagger} \leftarrow$ set every cells in $S$ to \texttt{NULL} in $\mathcal{D}$ \\
\Return $\mathcal{D}^{\dagger}$
\end{algorithm}

\subsection{Proof of Theorem~\ref{thm:batch-hard}}
\label{sec:batch-hard-appendix}
% \begin{theorem}
%     For every $d \geq 3$, $\poisoncrnbc$ is $\NP$-hard on datasets with $d$ dimensions.
% \end{theorem}
\begin{proof}
To show membership in \textbf{NP}, note that we may guess an incomplete dataset $\idata$ from the input complete dataset $\mathcal{D}$, and verify in polynomial time that all test points are certifiably non-robust over $\idata$, using Algorithm~\ref{alg:improved_decision_miss}.

For \textbf{NP}-hardness, we present a reduction from the \textsf{VertexCover} problem on $d$-regular graphs: Given a $d$-regular graph $G$ with vertex set $V$ and edge set $E$, in which all vertices have degree $d$, find a set $S \subseteq V$ of minimum size such that every edge in $E$ is adjacent to some vertex in $S$. 
    
    Let $V = \{1,2,\dots, n\}$ and $m = |E|$.  Without loss of generality, we assume that $G$ is not a clique. 
    By the Brook's theorem \cite{brooks1941colouring}, the $d$-regular graph $G$ is $d$-colorable. Let $\chi: V \rightarrow \{A_1, A_2, \dots, A_d\}$ be a $d$-coloring of $G$, which can be found in linear time.
    
    \textbf{Constructing $m$ test points.}
    We first construct $m$ test points on attributes $A_1, A_2,\dots, A_d$ as follows: For each edge $\{u, v\} \in E$,
    %  let $C_{i,j}$ be the unique color in $\{R, G, B\} \setminus \{\chi(i), \chi(j)\}$,  
    % we introduce a test point $\test_{i,j}$ with data values $i$, $j$ and $\langle i, j \rangle$ at attributes $\chi(i)$, $\chi(j)$ and $C_{i,j}$ respectively.
    introduce a test point $\test_{u,v}$ with values $u$ and $v$ at attributes $\chi(u)$ and $\chi(v)$, and a fresh integer value for all other attribute $A_k$ where $k \in \{1,2,\dots,d\}\setminus\{\chi(u), \chi(v)\}$. 
    
    It is easy to see that for two distinct edges $e=\{u,v\}$, $e'=\{u', v'\} \in E$, if $e$ and $e'$ do not share vertices, then $\test_{u,v}$ and $\test_{u',v'}$ do not agree on any attribute; and if they share a vertex $w$, then $\test_{u,v}$ and $\test_{u',v'}$ would agree on the attribute $\chi(w)$ with value $w$.
    
    Note that to construct the test points, we used $n + (d-2)m$ domain values, $n$ for each vertex $V$ and $(d-2)m$ for all the fresh constants.
    
    \textbf{Constructing a clean dataset.}
    For an attribute $A_k$, a domain value $u$ and a label $l$, we denote $p(A_k, u, l)$ as a datapoint 
    $$p(A_k, u, l) := (\square, \cdots, u, \cdots, \square),$$
    with label $l$, in which the attribute $A_k$ has domain value $u$, and each $\square$ denotes a fresh domain value.

    Let $M = n + (1 + n/d)^{3} + 1$ and $N = n + (d-2)m\cdot M$.
    We construct a dataset $\mathcal{D}$ of $2N$ points, in which $N$ points have label $l_1$, and $N$ points have label $l_2$ as follows.
    \begin{itemize}
    \item Datapoints with label $l_1$: for each domain value $u$ at attribute $A_k$, if $u \in V$, we introduce $p(A_k, u, l_1)$ once (and in this case, $A_k = \chi(u)$); or otherwise, we introduce $p(A_k, u, l_1)$ $M$ times.
    It is easy to see that there are $N = n + (d-2)m\cdot M$ datapoints with label $l_1$.
    \item Datapoints with label $l_2$: for each domain value $u$ at attribute $A_k$, we introduce $p(A_k, u, l_2)$ once, and we introduce $N - M$ fresh datapoints of the form
    $$(\square, \square, \cdots, \square).$$
    \end{itemize}
    
    This construction can be done in polynomial time of $m$ and $n$. 
    By construction, for each testpoint $\test_{u,v} = (c_1, c_2, \dots, c_d)$, if $c_k \in V$, then the domain value $c_k$ occurs exactly once as a domain value of $A_k$, or otherwise $M$ times among all datapoints with label $l_1$; and $c_k$ occurs exactly once as a domain value of $A_k$ among all datapoints with label $l_2$.
    
    Hence for each test point $\test_{u,v}$, $f_{\mathcal{D}}(\test_{u,v}) = l_1$, by noting that
    \begin{align*}
    \Pr(l_1 \mid \test_{u,v})_{\mathcal{D}} 
    &= \frac{1}{\Pr(\test_{u,v})_{\mathcal{D}}} \cdot \frac{1\cdot1\cdot M^{d-2}}{N^d} \cdot \frac{N}{2N} \\
    &> \frac{1}{\Pr(\test_{u,v})_{\mathcal{D}}} \cdot \frac{1\cdot 1 \cdot 1^{d-2}}{N^d} \cdot \frac{N}{2N} \\
    &= \Pr(l_2 \mid \test_{u,v})_{\mathcal{D}}.
    \end{align*}
    
    Now we argue that there is a vertex cover $S$ of $G$ with size at most $k$ if and only if we may alter at most $k$ missing values in $\mathcal{D}$ to obtain $\mathcal{D}'$ such that for each point $\test_{u,v}$,
    $$\Pr(l_1 \mid \test_{u,v})_{\mathcal{D}'} < \Pr(l_2 \mid \test_{u,v})_{\mathcal{D}'},$$
    which yields $f_{\mathcal{D}}(\test_{u,v}) \neq f_{\mathcal{D}'}(\test_{u,v})$.

    \framebox{$\Longrightarrow$}
    Let $S\subseteq V$ be a vertex cover of $G$.
    Consider the dataset $\mathcal{D}'$ obtained by altering for each $u \in S$ with $A_k = \chi(u)$,
    the datapoint
    $$p(A_k, u, l_1) = (\square, \cdots, u, \cdots, \square)$$
    into 
    $$(\square, \cdots, u', \cdots, \square),$$
    where $u' \notin V$.
    For every test point 
    $\test_{u,v}$, 
    we argue that either $u$ does not occur at all at attribute $\chi(u)$, or $v$ does not occur at all at attribute $\chi(v)$ in $\mathcal{D}'$. 
    Indeed, since $S$ is a vertex cover, either $u\in S$ or $v\in S$, and thus either $p(\chi(u), u, l_1)$ or $p(\chi(v), v, l_1)$ is altered so that either $u$ or $v$ does not occur at all at attributes $\chi(u)$ or $\chi(v)$ in $\mathcal{D}'$.
    Hence NBC would estimate that
    \begin{align*}
    \Pr(l_1 \mid \test_{u,v})_{\mathcal{D}'} = 0 < \frac{1}{\Pr(\test_{u,v})_{\mathcal{D}'}} \cdot \frac{1}{N^d} \cdot \frac{N}{2N} = \Pr(l_2 \mid \test_{u,v})_{\mathcal{D}'},
    \end{align*}
    as desired.
    
    \framebox{$\Longleftarrow$}
    Assume that there is a dataset $\mathcal{D}'$ with at most $k$ alternations to $\mathcal{D}$ such that 
    $f_{\mathcal{D}}(\test_{u,v}) = l_1$ but $f_{\mathcal{D}'}(\test_{u,v}) = l_2$ for every test point $\test_{u,v}$.
    It suffices to show that $G$ has a vertex cover of size at most $k$. Since $G$ must have a vertex cover of size $k \geq n$, we assume that $k < n$.
    
    Consider an arbitrary test point $\test_{u, v} = (c_1, c_2, \dots, c_n).$
    We have that 
    $$\Pr(l_1 \mid \test_{u, v})_{\mathcal{D}'} < \Pr(l_2 \mid \test_{1,2})_{\mathcal{D}'}.$$
    Let $x_i$ and $y_i$ be the absolute value of the change of frequency from $\mathcal{D}$ to $\mathcal{D}'$ of value $c_i$ for attribute $A_i$ in test points with labels $l_1$ and $l_2$ respectively.
    Then we have
    \begin{equation}
    \label{ineq:main}
    0 \leq (1 - x_i)(1-x_j)\cdot \prod_{k \in \{1,2,\dots, d\} \setminus \{i,j\}} (M-x_k) < \prod_{1 \leq k \leq d}(1+y_k).
    \end{equation}
    
    We argue that without loss of generality, we can assume all $x_i$ and $y_i$ are nonnegative: if $y_i$ is negative, then right-hand-side of (Eq.~\ref{ineq:main}) is $0$ and cannot hold; and if $x_i$ is negative, i.e., the alternations increase the frequency of certain attributes among points with label $l_1$, then not performing such alterations would also preserve the inequality (Eq.~\ref{ineq:main}) with less than $k$ alterations.
    
    Note that $x_1+x_2+\dots + x_d + y_1 + y_2 + \dots + y_d \leq k$.
    
    For each test point $\test_{u,v}$, let $A_i = \chi(u)$ and $A_j = \chi(v)$, and we argue that we must have $x_j = 1$ or $x_j = 1$ (or both). 
    Suppose for contradiction that $x_i = x_j = 0$. Then, we have 
    \begin{align*}
    &\prod_{k \in \{1,2,\dots, d\} \setminus \{i,j\}} (M-x_k) \\
    &= (1 - x_i)(1-x_j)\cdot \prod_{k \in \{1,2,\dots, d\} \setminus \{i,j\}} (M-x_k) \\
    &< \prod_{1 \leq k \leq d}(1+y_k) \\
    &\leq (1+\frac{y_1 + y_2 + \dots + y_d}{d})^d \\
    &\leq (1+\frac{n}{d})^d.
    \end{align*}
    On the other hand, we have for each $k \in \{1,2,\dots, d\} \setminus \{i,j\}$, $x_k \leq n$, and thus 
    $$M - x_k \geq M - n > (1+\frac{n}{d})^{3} \geq (1+\frac{n}{d})^{\frac{d}{d-2}},$$
    and it follows that 
    $$\prod_{k \in \{1,2,\dots, d\} \setminus \{i,j\}} (M-x_k) \geq (1+\frac{n}{d})^{d},$$
    which is a contradiction.
    
    Now consider the set 
    $$S = \{u \mid A_k = \chi(u), x_k = 1\}.$$ 
    The set $S$ is a vertex cover, since for every edge $\{u,v\}\in E$, if the test point $\test_{u,v}$ has 
    $A_i = \chi(u)$ and $A_j=\chi(v)$, we must have $x_i = 1$ or $x_j = 1$, and thus either $u\in S$ or $v\in S$.
    Since each $x_k = 1$ corresponds to an alteration, the size of $S$ is at most $k$.
    
    The proof concludes by noting that both $\mathcal{D}$ and $\mathcal{D}'$ are possible worlds of the incomplete dataset $\mathcal{D}^{\dagger}$ obtained by setting the altered cells in $\mathcal{D}'$ from $\mathcal{D}$ as missing cells.
\end{proof}

\subsection{Additional Examples}

\begin{figure*}
\centering
\subfloat[A possible world $\mathcal{D}$]{
\adjustbox{valign=b}{
\begin{tabular}{cc}
\begin{tabular}{c | c | c }
$X$ & $Y$ &  \\
\hline
$a$ & $b$ & \multirow{4}{1pt}{$l^*$}\\
$a$ & $\bot_1$ \\
$\bot_2$ & $b$ \\
$\bot_3$ & $b$
\end{tabular}
&
\begin{tabular}{c | c | c}
$X$ & $Y$ &  \\
\hline
$a$ & $b$ & \multirow{5}{1pt}{$l$}\\
$\bot_1$ & $b$ & \\
$\bot_2$ & $\bot_3$ & \\
$\bot_4$ & $\bot_5$ & \\
$\bot_6$ & $\bot_7$ &
\end{tabular}
\end{tabular}
}
}
\quad
\subfloat[$\mathcal{D}_+$ and $\mathcal{D}_+'$, obtained by altering $\bot_1$ and $\bot_3$ to $a$ and $b$ in $\mathcal{D}$ respectively.]{
\adjustbox{valign=b}{
\begin{tabular}{cc}
\begin{tabular}{c | c | c }
$X$ & $Y$ &  \\
\hline
$a$ & $b$ & \multirow{4}{1pt}{$l^*$}\\
$a$ & $\bot_1$ \\
$\bot_2$ & $b$ \\
$\bot_3$ & $b$
\end{tabular}
&
\begin{tabular}{c | c | c}
$X$ & $Y$ &  \\
\hline
$a$ & $b$ & \multirow{5}{1pt}{$l$}\\
\st{$\bot_1$} $a$ & $b$ & \\
$\bot_2$ & \st{$\bot_3$} $b$ & \\
$\bot_4$ & $\bot_5$ & \\
$\bot_6$ & $\bot_7$ &
\end{tabular}
\end{tabular}
}
}
\quad
\subfloat[$\mathcal{D}_-$ and $\mathcal{D}_-'$, obtained by altering $a$ and $b$ to $c_1$ and $c_1'$ in $\mathcal{D}$ respectively.]{
\adjustbox{valign=b}{
\begin{tabular}{cc}
\begin{tabular}{c | c | c }
$X$ & $Y$ &  \\
\hline
$a$ & $b$ & \multirow{4}{1pt}{$l^*$}\\
\st{$a$} $c_1$ & $\bot_1$ \\
$\bot_2$ & $b$ \\
$\bot_3$ & \st{$b$} $c_1'$
\end{tabular}
&
\begin{tabular}{c | c | c}
$X$ & $Y$ &  \\
\hline
$a$ & $b$ & \multirow{5}{1pt}{$l$}\\
$\bot_1$ & $b$ & \\
$\bot_2$ & $\bot_3$ & \\
$\bot_4$ & $\bot_5$ & \\
$\bot_6$ & $\bot_7$ &
\end{tabular}
\end{tabular}
}
}
\caption{
For the test point $\test = (a, b)$, altering $\bot_1$ to $a$ in $\mathcal{D}_+$ (among other possible ways) achieves the most increase in $\support{\mathcal{D}}{l}{\test}$, and altering $a$ to $c_1$ in $\mathcal{D}_-$ (among other possible ways) achieves the most decrease in $\support{\mathcal{D}}{l^*}{\test}$.
}
\label{fig:observation}
\end{figure*}

\begin{example}
\label{ex:strategy}
Let us consider the dataset $\mathcal{D}$ in Figure~\ref{fig:observation}(a) and a test point $\test = (a, b)$.
We have that.
$$\support{\mathcal{D}}{l^*}{\test} = \frac{4}{9} \cdot \frac{2}{4} \cdot \frac{3}{4} \quad\text{ and }\quad \support{\mathcal{D}}{l}{\test} = \frac{5}{9} \cdot \frac{1}{5} \cdot \frac{2}{5},$$
and thus $\support{\mathcal{D}}{l^*}{\test} > \support{\mathcal{D}}{l}{\test}$.

\begin{itemize}

\item To increase $\support{\mathcal{D}}{l}{\test}$, we can (1) alter the value $\bot_2$ in attribute $X$ into $a$ to obtain a dataset $\mathcal{D}_+$, or (2) alter the value $\bot_3$ in attribute $Y$ into $b$ to obtain a dataset $\mathcal{D}_+'$ shown in Figure~\ref{fig:observation}(b).
Note that 
$$\support{\mathcal{D}_+}{l}{\test} = \frac{5}{9} \cdot \frac{2}{5} \cdot \frac{2}{5} \quad\text{ and }\quad \support{\mathcal{D}_+'}{l}{\test} = \frac{5}{9} \cdot \frac{1}{5} \cdot \frac{3}{5},$$
performing (1) achieves a bigger reduction in $\support{\mathcal{D}}{l^*}{\test}$ than (2), since for any test point $x = (x_1, x_2)$ in $\mathcal{D}$, $\Pr(x_1 = a \mid l)_{\mathcal{D}} < \Pr(x_2 = b \mid l)_{\mathcal{D}}$.

\item To decrease $\support{\mathcal{D}}{l^*}{\test}$, we can (1) alter the value $a$ in attribute $X$ into some $c_1$ to obtain a dataset $\mathcal{D}_-$, or (2) alter the value $b$ in attribute $Y$ into some $c_1'$ to obtain a dataset $\mathcal{D}_-'$ shown in Figure~\ref{fig:observation}(c).
Note that 
$$\support{\mathcal{D}_-}{l^*}{\test} = \frac{4}{9} \cdot \frac{1}{4} \cdot \frac{3}{4} \quad\text{ and }\quad \support{\mathcal{D}_-'}{l^*}{\test} = \frac{4}{9} \cdot \frac{2}{4} \cdot \frac{2}{4},$$
performing (1) achieves a bigger reduction in $\support{\mathcal{D}}{l^*}{\test}$ than (2), since for any test point $x = (x_1, x_2)$ in $\mathcal{D}$, $ \Pr(x_1 = a \mid l^*)_{\mathcal{D}} < \Pr(x_2 = b \mid l^*)_{\mathcal{D}}$.
\end{itemize}
\end{example}

\begin{example}
\label{ex:same-step}
For $\mathcal{D}_+$ in Example~\ref{ex:strategy}, we have 
$$\support{\mathcal{D}_+}{l}{\test} = \frac{5}{9} \cdot \frac{2}{5} \cdot \frac{2}{5}.$$ 
In this case, altering $\bot_2$ to $a$ or $\bot_5$ to $b$ achieves the largest increase in $\support{\mathcal{D}_+}{l}{\test}$, since both values have the smallest frequency. Suppose that we alter $\bot_2$ to $a$ and obtain dataset $\mathcal{D}_{++}$. We have 
$$\support{\mathcal{D}_{++}}{l}{\test} = \frac{5}{9} \cdot \frac{3}{5} \cdot \frac{2}{5}.$$
Then, we would alter $\bot_5$ to $b$, and achieve
$$\support{\mathcal{D}_{+++}}{l}{\test} = \frac{5}{9} \cdot \frac{3}{5} \cdot \frac{3}{5}.$$
Hence O1 holds in this case, since
\begin{align*}
\support{\mathcal{D}_+}{l}{\test} - \support{\mathcal{D}}{l}{\test} 
&= \support{\mathcal{D}_{++}}{l}{\test} - \support{\mathcal{D}_{+}}{l}{\test} \\
&< \support{\mathcal{D}_{+++}}{l}{\test} - \support{\mathcal{D}_{++}}{l}{\test}
\end{align*}

For $\mathcal{D}_-$ in Example~\ref{ex:strategy}, we have 
$$\support{\mathcal{D}_-}{l^*}{\test} = \frac{4}{9} \cdot \frac{1}{4} \cdot \frac{3}{4}.$$ 
Different from $\mathcal{D}_+$, further altering the only $a$ in attribute $X$ to some other value to obtain $\mathcal{D}_{--}$ still achieves the largest decrease in $\support{\mathcal{D}_-}{l^*}{\test}$, since the frequency of $a$ is still the smallest. We have
$$\support{\mathcal{D}_{--}}{l^*}{\test} = \frac{4}{9} \cdot \frac{0}{4} \cdot \frac{3}{4}.$$
Hence O2 holds in this case, since
\begin{align*}
\support{\mathcal{D}_-}{l}{\test} - \support{\mathcal{D}}{l}{\test} 
&= \support{\mathcal{D}_{--}}{l}{\test} - \support{\mathcal{D}_{-}}{l}{\test}
\end{align*}
\end{example}

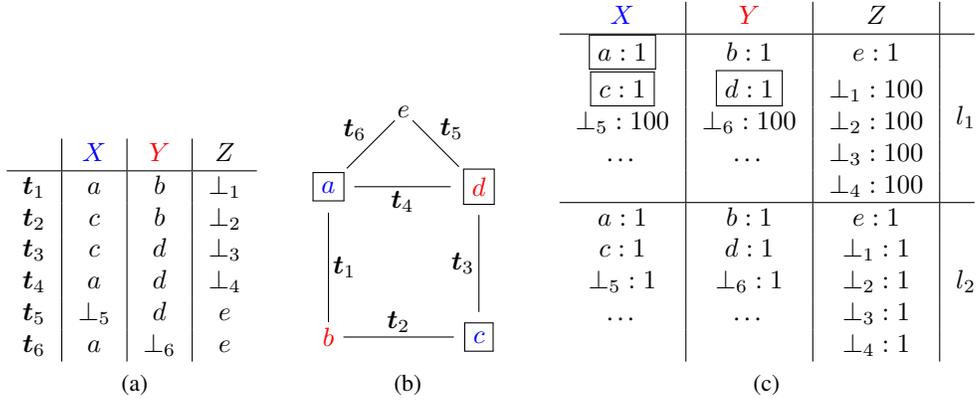
\begin{figure*}
\centering
\vspace{-1em}
\subfloat[]{
\adjustbox{valign=b}{\begin{tabular}{c | c | c | c}
& $\textcolor{blue}{X}$ & $\textcolor{red}{Y}$ & $Z$ \\
\hline
$\test_1$ & $a$& $b$& $\bot_1$\\
$\test_2$ & $c$& $b$& $\bot_2$\\
$\test_3$ & $c$& $d$& $\bot_3$\\
$\test_4$ & $a$& $d$& $\bot_4$\\
$\test_5$ & $\bot_5$ & $d$ & $e$\\
$\test_6$ & $a$ & $\bot_6$ & $e$ 
\end{tabular}
}
}
\quad
\subfloat[]{
\adjustbox{valign=b}{\begin{tikzpicture}[auto=left,vnode/.style={circle,black,inner sep=1pt,scale=1},el/.style = {inner sep=3/2pt}]
\node[vnode,red] (s1) at (0, 0) {$b$};
\node[vnode] (s2) at (2, 0) {\framebox{\textcolor{blue}{$c$}}};
\node[vnode] (s3) at (2, 2) {\framebox{\textcolor{red}{$d$}}};
\node[vnode] (s4) at (1, 3) {$e$};
\node[vnode] (s5) at (0, 2) {\framebox{\textcolor{blue}{$a$}}};
\path[-] (s1) edge[el] node {$\test_2$} (s2);
\path[-] (s2) edge[el] node {$\test_3$}  (s3);
\path[-] (s3) edge[el] node {$\test_4$}  (s5);
\path[-] (s5) edge[el] node {$\test_1$}  (s1);
\path[-] (s5) edge[el] node {$\test_6$}  (s4);
\path[-] (s4) edge[el] node {$\test_5$}  (s3);
\end{tikzpicture}
}
}
\quad
\subfloat[]{
\adjustbox{valign=b}{
\begin{tabular}{c | c | c | c}
$\textcolor{blue}{X}$ & $\textcolor{red}{Y}$ & $Z$ &\\
\hline
\framebox{$a: 1$}& $b: 1$& $e: 1$ & \multirow{5}{1pt}{$l_1$}\\
\framebox{$c: 1$}& \framebox{$d: 1$}& $\bot_1: 100$ & \\
$\bot_5: 100$& $\bot_6: 100$& $\bot_2: 100$ & \\
\dots & \dots & $\bot_3: 100$ & \\
& & $\bot_4: 100 $ & \\
\hline
$a: 1$& $b: 1$& $e: 1$ & \multirow{5}{1pt}{$l_2$}\\
$c: 1$& $d: 1$& $\bot_1: 1$ & \\
$\bot_5: 1$& $\bot_6: 1$& $\bot_2: 1$ & \\
\dots & \dots & $\bot_3: 1$ & \\
& & $\bot_4: 1 $ &
\end{tabular}
}
% \begin{tabular}{c | c | c | c}
% $X$ & $Y$ & $Z$ &\\
% \hline
% $a$& $b$& $\bot$ & \multirow{3}{1pt}{$l_1$}\\
% $c$& $d$& $\bot$ & \\
% $\bot$& $\bot$& $e$ & \\
% $\bot$& $\bot$& $1$ & \\
% $\bot$& $\bot$& $1$ & \\
% $\dots$ & & &  \\
% $\bot$& $\bot$& $2$ & \\
% $\bot$& $\bot$& $2$ & \\
% $\dots$ & & &  \\
% $\bot$& $\bot$& $3$ & \\
% $\bot$& $\bot$& $3$ & \\
% $\dots$ & & &  \\
% $\bot$& $\bot$& $4$ & \\
% $\bot$& $\bot$& $4$ & \\
% $\dots$ & & &  \\
% $5$& $\bot$& $\bot$ & \\
% $5$& $\bot$& $\bot$ & \\
% $\dots$ & & &  \\
% $\bot$& $6$& $\bot$ & \\
% $\bot$& $6$& $\bot$ & \\
% $\dots$ & & &  \\
% \hline
% $a$& $b$& $\bot$ & \multirow{3}{1pt}{$l_2$}\\
% $c$& $d$& $\bot$ & \\
% $\bot$& $\bot$& $e$ & \\
% $\bot$& $\bot$& $1$ & \\
% $\bot$& $\bot$& $2$ & \\
% $\bot$& $\bot$& $3$ & \\
% $\bot$& $\bot$& $4$ & \\
% $5$& $\bot$& $\bot$ & \\
% $\bot$& $6$& $\bot$ & \\
% \end{tabular}
}
% \vspace{-1em}
\caption{
The test points are shown in (a). 
The graph in (b) captures the correlations among the test points in (a).
A dataset is shown in (c).
Poisoning the minimum number of cells in (c) such that all test points in (a) are not certifiably robust can be simulated by finding a minimum vertex cover, marked by boxes, in the graph in (b).
% \xiating{TODO: fix the formatting\dots}
}
\label{fig:graph}
\end{figure*}

\begin{example}
\label{ex:graph}
Consider the $6$ test points shown in Figure~\ref{fig:graph}(a).
Note that since the test points $\test_1$, $\test_4$ and $\test_6$ agree on the attribute value $a$ for attribute $X$, the change in the relative frequency of the attribute value $a$ in the datasets affects the support values (and thus the predictions) to all $3$ test points simultaneously.
The graph in Figure~\ref{fig:graph}(b) captures precisely this correlation among the test points, in which each vertex represents a value from an attribute, and each edge represents a test point containing those two values at their corresponding attributes. The $3$ attributes $X$, $Y$ and $Z$ are distinguished by a $3$-coloring of the graph. 

Consider a complete dataset $\mathcal{D}$ containing $2N$ datapoints shown in Figure~\ref{fig:graph}(c), in which $N$ datapoints have label $l_1$ and $N$ datapoints have label $l_2$ for some sufficiently large $N$. 
For each attribute $A\in\{X, Y, Z\}$, the entry $v : n$ with label $l\in\{l_1, l_2\}$ indicates that there are $n$ datapoints with value $v$ for attribute $A$ in $\mathcal{D}$. 

It is easy to see that for each test point $\test_i$ in Figure~\ref{fig:graph}(a), 
$$f_{\mathcal{D}}(\test_i) = l_1,$$ 
since the relative frequencies of each $\bot_i$ among datapoints with label $l_1$ is much higher than those with label $l_2$.

Note that if we alter $\mathcal{D}$ to decrease the frequency of value $a$ in attribute $X$ in label $l_1$, this single alteration causes the support values of $l_1$ for test points $\test_1$, $\test_4$ and $\test_6$ to drop to $0$, and flips the predictions of all three test points to $l_2$. 
Therefore, the minimum number of alterations in the datasets corresponds to a minimum vertex cover in the graph in Figure~\ref{fig:graph}(b).
\end{example}

\section{Additional Experimental Results in Section~\ref{sec:experiments}}
\label{sec:more-experiments-appendix}

\subsection{Datasets}
\label{sec:datasets-appendix}
We use ten real-world datasets from Kaggle~\cite{web:kaggle}: \textsf{heart (HE)}~\cite{dat:heart}, \textsf{fitness-club (FC)}~\cite{dat:fitnessclub}, \textsf{fetal-health (FH)}~\cite{dat:fetalhealth}, \textsf{employee (EM)}~\cite{dat:employee}, \textsf{winequalityN (WQ)}~\cite{dat:winequality}, \textsf{company-bankruptcy (CB)}~\cite{dat:companybankruptcy}, \textsf{Mushroom (MR)}~\cite{dat:mushroom}, \textsf{bodyPerformance (BP)}~\cite{dat:bodyPerformance}, \textsf{star-classification (SC)}~\cite{dat:starclassification}, \textsf{creditcard (CC)}~\cite{dat:creditcard}. The metadata of our datasets are summarized in Table~\ref{tab:datasets}.
\begin{table}[!ht]
    \centering
    \begin{tabular}{c|c|c|c}
      \toprule
      Dataset &  \# Rows & \# Features & \# Labels \\ 
      \hline
      \textsf{heart (HE)} & 918 & 12 & 2 \\ 
      \textsf{fitness-club (FC)} & 1,500 & 8 & 2 \\
      \textsf{fetal-health (FH)} & 2,126 & 22 & 3 \\
      \textsf{employee (EM)} & 4,653 & 9 & 2 \\ 
      \textsf{winequalityN (WQ)} & 6,497 & 13 & 2 \\
      \textsf{Company-Bankruptcy (CB)} & 6,819 & 96 & 2 \\
      \textsf{Mushroom (MR)} & 8,124 & 23 & 2 \\
      \textsf{bodyPerformance (BP)} & 13,393 & 12 & 4 \\ 
      \textsf{star-classification (SC)} & 100,000 & 18 & 3 \\
      \textsf{creditcard (CC)} & 568,630 & 31 & 2 \\
      \bottomrule
    \end{tabular}
    \caption{Metadata of the datasets.} 
    \label{tab:datasets}
\end{table}

\subsection{Additional Results for Decision Problem}
\label{sec:more-results-decision-appendix}
In this section, we provide additional results for the decision problem in Figure~\ref{fig:decision_time_vs_missing_rate_appendix} and Figure~\ref{fig:decision_time_vs_number_of_test_points_appendix}.
\begin{figure*}[!ht]
    \centering
    \includegraphics[width=0.96\textwidth]{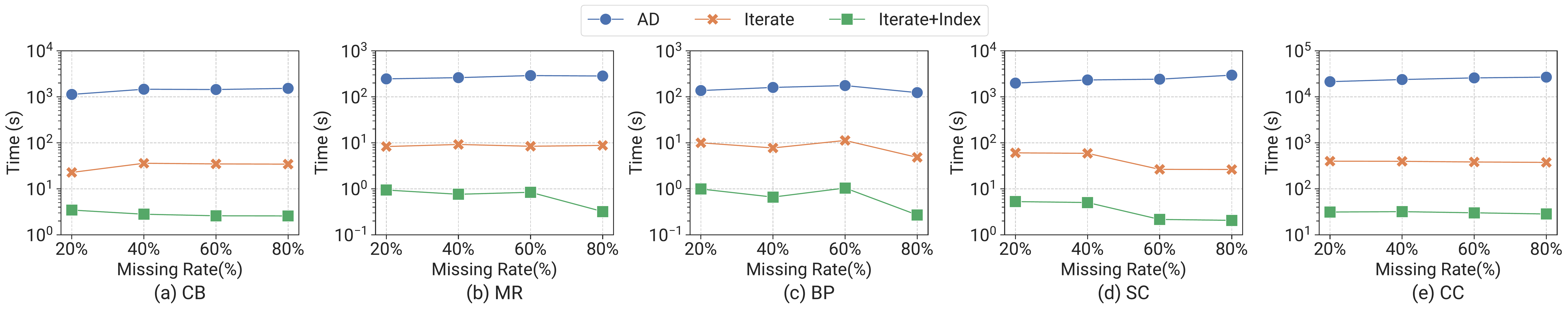}
    \caption{Decision - Running Time vs Missing Rates on Different Datasets}
    \label{fig:decision_time_vs_missing_rate_appendix}
\end{figure*}

\begin{figure*}[!ht]
    \centering
    \includegraphics[width=0.96\textwidth]{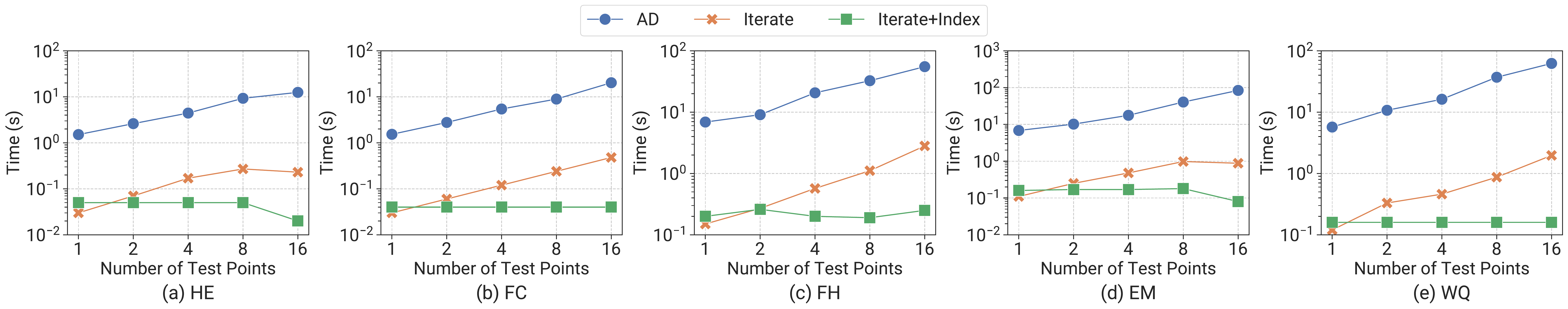}
    \caption{Decision - Running Time vs Number of Test Points on Different Datasets}
    \label{fig:decision_time_vs_number_of_test_points_appendix}
\end{figure*}

\subsection{Additional Results for Poisoning Problem}
\label{sec:more-results-poisoning-appendix}
\begin{figure*}[!ht]
    \centering
    \includegraphics[width=0.96\textwidth]{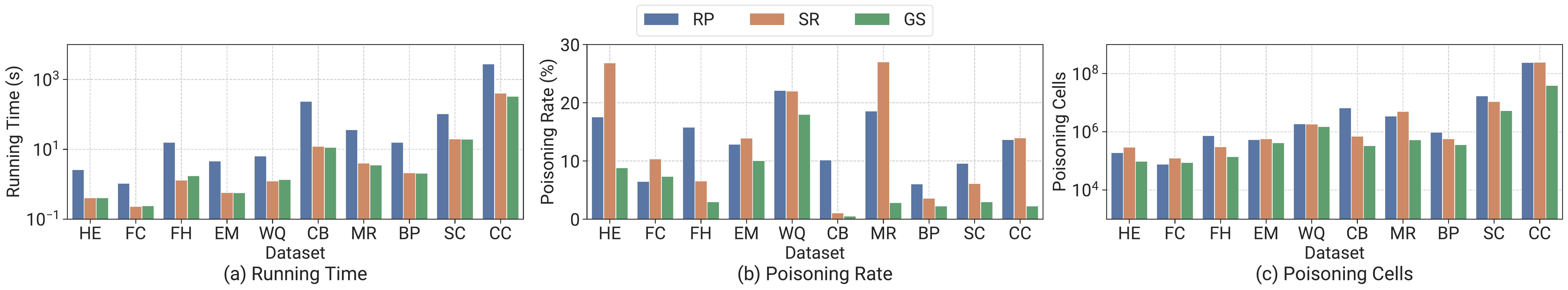}
    \caption{Multiple Points Data Poisoning---Running Time and Poisoning Rate on Different Datasets}
    \label{fig:multiple_point_poisoning}
\end{figure*}

We show the additional results for the poisoning problem in Figure~\ref{fig:multiple_point_poisoning}. As for the \textsf{Multiple Points Data Poisoning Problem}, \textbf{GS} is more effective than \textbf{RP} and \textbf{SR} over most datasets. Note that despite \textbf{SR} and \textbf{GS} have similar running times but their poisoning rate varies: This implies that the optimal strategies A1 and A2 mentioned in Section~\ref{sec:poisoning} are efficient to execute and the optimal sequence of applying A1 and A2 can minimize the number of poisoned cells drastically. Furthermore, Figure~\ref{fig:multiple_point_poisoning} shows the exact number of poisoned cells. Note that the number of poisoning cells is not small in Figure~\ref{fig:multiple_point_poisoning}, which indicates that the answer to the \textsf{Decision Problem} can help us to accelerate the data cleaning process.

\end{document}